\documentclass[sigconf]{acmart}

\usepackage{subfigure}
\usepackage{pifont}
\usepackage{bbding}
\usepackage{threeparttable}
\usepackage{colortbl}
\usepackage{graphicx}
\usepackage{makecell}
\usepackage{bbm}

\usepackage{booktabs}
\usepackage{multirow}
\usepackage{amsmath}
\usepackage{adjustbox}
\usepackage{xcolor}

\usepackage{enumitem}

\newcommand{\res}[2]{#1{\ensuremath{{\scriptscriptstyle \pm #2}}}}

\definecolor{rankone}{HTML}{FFE599}   
\definecolor{ranktwo}{HTML}{D9EAD3}   
\definecolor{rankthree}{HTML}{D0E0E3} 

\newcommand{\best}[1]{\cellcolor{rankone}#1}
\newcommand{\second}[1]{\cellcolor{ranktwo}#1}
\newcommand{\third}[1]{\cellcolor{rankthree}#1}

\DeclareRobustCommand{\bestcap}{\begingroup\setlength{\fboxsep}{1pt}\colorbox{rankone}{best}\endgroup}
\DeclareRobustCommand{\secondcap}{\begingroup\setlength{\fboxsep}{1pt}\colorbox{ranktwo}{second-best}\endgroup}
\DeclareRobustCommand{\thirdcap}{\begingroup\setlength{\fboxsep}{1pt}\colorbox{rankthree}{third-best}\endgroup}

\newcommand{\vpara}[1]{\vspace{0.05in}\noindent\textbf{#1 }}

\usepackage{amsthm}

\newtheorem{theorem}{Theorem}
\newtheorem{proposition}{Proposition}
\newtheorem{remark}{Remark}

\AtBeginDocument{%
  }

\setcopyright{acmlicensed}
\copyrightyear{2018}
\acmYear{2018}
\acmDOI{XXXXXXX.XXXXXXX}
\acmConference[Conference acronym 'XX]{Make sure to enter the correct
  conference title from your rights confirmation email}{June 03--05,
  2018}{Woodstock, NY}
\acmISBN{978-1-4503-XXXX-X/2018/06}

\acmSubmissionID{123-A56-BU3}

\begin{document}

\title{Towards One-for-All Foundation Model for\\ Attributed Graph Clustering}

\author{Yunhui Liu}
\orcid{0009-0006-3337-0886}
\affiliation{%
  \institution{State Key Laboratory for Novel Software Technology\\ Nanjing University}
  \city{Nanjing}
  \country{China}}
\email{lyhcloudy1225@gmail.com}

\author{Xudong Jin}
\author{Kang Zhang}
\affiliation{%
  \institution{State Key Laboratory for Novel Software Technology\\ Nanjing University}
  \city{Nanjing}
  \country{China}}

\author{Danshuo An}
\author{Yu Xing}
\affiliation{%
  \institution{State Key Laboratory for Novel Software Technology\\ Nanjing University}
  \city{Nanjing}
  \country{China}}

\author{Te Song}
\affiliation{%
  \institution{State Key Laboratory for Novel Software Technology\\ Nanjing University}
  \city{Nanjing}
  \country{China}}

\author{Jia Liu}
\affiliation{%
  \institution{State Key Laboratory for Novel Software Technology\\ Nanjing University}
  \city{Nanjing}
  \country{China}}

\author{Tieke He}
\orcid{0000-0001-9649-1796}
\affiliation{%
  \institution{State Key Laboratory for Novel Software Technology\\ Nanjing University}
  \city{Nanjing}
  \country{China}}
\email{hetieke@gmail.com}

\renewcommand{\shortauthors}{Liu et al.}

\begin{abstract}
Attributed graph clustering aims to discover node groups by jointly exploiting node attributes and graph topology, yet its unsupervised nature makes model selection and adaptation inherently difficult. 
Existing methods typically train and tune a separate model for each input graph, leading to costly and fragile pipelines that often fail to transfer across graphs with different feature spaces, structural patterns, and attribute-structure correlations. 
In this paper, we study a one-for-all alternative: can a single model be trained once and directly applied to diverse attributed graphs without graph-specific training, fine-tuning, or hyperparameter search? 
We propose \texttt{OFAG}, a foundation model for attributed graph clustering. 
Building upon Prior-data Fitted Networks, \texttt{OFAG} learns a reusable clustering inference strategy from synthetic attributed graphs generated under broad priors over latent clusters, node attributes, and graph structures. 
To handle incompatible feature spaces across graphs, \texttt{OFAG} adopts a dimension-agnostic signal-wise graph encoder that treats each feature channel as a graph signal and models its response to shared graph filters. 
The model is trained with a hyperspherical clustering objective, producing clustering-friendly node representations in a single forward pass at inference time. 
On ten datasets, one frozen \texttt{OFAG} model achieves the best mean performance and average rank across NMI, ACC, ARI, and F1, while completing all ten datasets in 12.43 minutes total---over $6\times$ faster than the second-fastest baseline and nearly $28\times$ faster than the second-best on clustering quality.
Our code and pretrained checkpoint are available at \url{https://github.com/Cloudy1225/OFAG}, allowing practitioners to directly apply \texttt{OFAG} to their own attributed graph datasets without additional training or tuning.
\end{abstract}


\maketitle

\section{Introduction}
\label{sec:intro}

Foundation models have transformed machine learning workflows in language, vision, and now supervised graph tasks: pre-train once on broad data, then apply to diverse inputs without task-specific optimization~\citep{GFMSurvey, GFMSurvey1}.
Attributed graph clustering, which partitions nodes by jointly exploiting features and topology without labels~\cite{DGCSurvey, AGCSurvey}, has yet to benefit from this shift.
Today, every new graph still demands the same costly ritual: design or select a model, tune hyperparameters blindly on unlabeled data, train from scratch, and hope that choices calibrated on one graph transfer to the next.

This per-graph workflow is expensive for a principled reason: real attributed graphs vary substantially and unpredictably.
Feature dimensions range from tens to thousands; attributes may be dense, continuous, or bag-of-words; graph structures span homophilous, heterophilous, and weakly-informative connectivity regimes; and cluster sizes and degree distributions may be highly imbalanced.
As a result, assumptions that work well on one graph can fail on another. 
For example, aggressively propagating features along edges can be beneficial when graph structure agrees with clusters, but harmful when edges connect nodes from different groups.
Without any validation labels, there is no signal to detect such failures or to guide model selection, and the cost of blind hyperparameter search across these axes compounds with every new graph.

Graph foundation models (GFMs) represent a natural starting point, yet existing work leaves attributed graph clustering largely unaddressed~\cite{GFMPosition, GFMSurvey, GFMSurvey1}.
\emph{Text-space GFMs}~\cite{OFA, GraphGPT, GOFA} unify graphs through natural language but require textual node attributes and dataset-specific templates.
\emph{Feature-aligned GFMs}~\cite{GCOPE, SAMGPT, AnomalyGFM} extend transfer to non-textual graphs but typically require target-graph fine-tuning.
\emph{Fully inductive GFMs}~\cite{GraphAny, NodePFN, GraphPFN} avoid per-graph training but rely on a labeled context set at inference to anchor predictions, a resource that clustering cannot provide.
Beyond the availability of labels, attribution graph clustering poses three tightly coupled challenges that distinguish it from all supervised transfer settings:
(i) \emph{permutable and variable cluster identities}: cluster labels carry no consistent meaning across graphs and the number of clusters $K$ varies, ruling out a shared output head;
(ii) \emph{no inference-time supervision}: there are no labels at test time, no validation set for model selection, and no few-shot context to adapt to the target graph;
(iii) \emph{unsupervised feature space bridging}: heterogeneous feature spaces are a known challenge for cross-graph transfer, but supervised GFMs address this through labeled fine-tuning or adaptation. Clustering must bridge incompatible feature dimensions and semantics entirely without downstream supervision.
Together, these make zero-shot unsupervised transfer genuinely hard in a way that simply repurposing existing GFMs cannot resolve.

In this work, we ask whether attributed graph clustering can be made \emph{one-for-all}: can a single model be trained once and then directly applied to arbitrary attributed graphs without graph-specific training, fine-tuning, or hyperparameter search?
We propose \texttt{OFAG}, a \underline{O}ne-\underline{f}or-\underline{A}ll foundation model for attributed \underline{G}raph clustering, trained once on synthetic attributed graphs and applied to any new graph in a single forward pass, after which a standard algorithm recovers the clusters.

The key idea is to adapt Prior-data Fitted Networks (PFNs)~\cite{PFNPosition, TabPFNv2} to the unsupervised graph setting.
PFNs amortize Bayesian inference over a prior on tasks, learning a reusable inference strategy from synthetic data rather than fitting parameters to each new instance.
Adapting this paradigm to clustering is non-trivial: unlike supervised PFNs where labeled context nodes ground predictions, \texttt{OFAG} must recover cluster structure from attributes and topology alone, with no context labels and a variable, unlabeled query set.
We address challenges (i) and (ii) by pre-training \texttt{OFAG} on synthetic attributed graphs drawn from the \emph{Coupled Attributed Graph Mixture} (CAGM) prior, which generates clustering tasks with variable $K$ and permuted cluster identities, covering diverse feature types, cluster separabilities, structural regimes, and attribute-topology coupling strengths, compelling the model to learn an adaptive inference strategy without relying on any labeled context.

To address challenge (iii), \texttt{OFAG} introduces the \emph{Signal-wise Filter Response Transformer} (SwiFRT).
Standard GNNs apply feature transformations tied to a fixed feature space, making them incompatible with zero-shot transfer across heterogeneous graphs.
SwiFRT instead treats each feature channel as a graph signal and encodes its \emph{structural response} under a bank of graph diffusion filters~\cite{GSP} (capturing whether a channel smooths uniformly, oscillates, or amplifies with local topology) using shared parameters applied independently to each channel.
This makes the model both dimension-agnostic and invariant to feature-channel ordering, while remaining sensitive to graph structure.
The model is trained with a hyperspherical clustering objective derived from the von Mises--Fisher distribution~\cite{vMF}, shaping the embedding space so that same-cluster nodes concentrate near shared prototypes and prototypes spread uniformly across the unit hypersphere.

We evaluate \texttt{OFAG} on ten benchmark datasets spanning diverse domains, from thousands to millions of nodes and covering both homophilous and heterophilous structures.
A single frozen \texttt{OFAG} model achieves the best mean score and the best average rank across all four standard metrics (NMI, ACC, ARI, and F1), outperforming all dataset-specific baselines.
Crucially, it is simultaneously the fastest method: \texttt{OFAG} completes all ten datasets in 12.43 minutes total, while the second-fastest baseline requires over $6\times$ more time with substantially lower clustering quality, and the second-best baseline on clustering quality requires nearly $28\times$ more computation.
Since dataset-specific pipelines must also be re-tuned on every new graph without labels, accounting for hyperparameter search pushes the practical time savings to orders of magnitude.
These results establish that prior-fitted foundation-model inference is not only conceptually appealing but also a practically superior route to fast and reusable attributed graph clustering.


\section{Related Work}
\label{sec:related}

\subsection{Attributed Graph Clustering}
\label{sec:rw-agc}

Attributed graph clustering (AGC) aims to partition nodes by jointly exploiting node attributes and graph topology. Driven by graph neural networks~\cite{GCN} and self-supervised learning~\cite{GSSLSurvey}, recent AGC methods can be broadly viewed through an encode--cluster--optimize pipeline~\cite{AGCSurvey}. 
For encoding, methods range from non-parametric graph filters that smooth node features over topology~\cite{AGC, SASE} to parametric GNN encoders that learn deep structure-aware representations~\cite{S3GC, BLNN}. 
For cluster assignment, existing approaches use neural prototypes~\cite{DAEGC, DinkNet}, self-expressive subspace coefficients~\cite{SAGSC, MS2CAG}, or differentiable community assignment modules~\cite{DMoN, Neuromap}. 
For optimization, objectives have evolved from two-stage embedding followed by $k$-means to joint training with graph reconstruction~\cite{GAE, DAEGC, DGAC}, contrastive learning~\cite{NS4GC, TIER, DGF}, feature decorrelation~\cite{CCASSG, DCRN, SSGE}, or clustering-specific losses~\cite{DEC, DinkNet, CoCo}.
However, existing methods remain dataset-specific, which is especially limiting without labels for validation, and assumptions tuned on one graph may not transfer across diverse feature types, structural regimes, and attribute-topology correlations. 
OFAG instead learns a reusable clustering inference strategy, enabling the same frozen model to cluster unseen graphs without graph-specific training, fine-tuning, or hyperparameter search.

\subsection{Graph Foundation Models}
\label{sec:rw-gfm}

Graph foundation models (GFMs) aim to generalize graph learning beyond individual datasets~\cite{GFMSurvey, GFMSurvey1}. 
Existing GFMs can be broadly grouped by how they handle cross-dataset heterogeneity. \emph{Text-space GFMs} unify graph datasets through natural language, using LLMs to encode textual node attributes or perform prediction in text space~\cite{GraphGPT, OFA, GOFA, ZeroG, GraphR1}. 
They support task transfer and few-shot learning, but often rely on textual attributes or dataset-specific graph-to-text templates. 
\emph{Feature-aligned GFMs} extend GFMs to non-textual graphs by mapping heterogeneous node features into a shared space, for example through SVD-based alignment or learnable patching~\cite{GCOPE, UniGraph, SAMGPT, PatchNet, AnomalyGFM}. 
These methods improve cross-dataset transfer, but typically still require downstream adaptation or fine-tuning. 
\emph{Fully inductive GFMs} further aim to operate on arbitrary graphs without target-graph training. Representative methods include GraphAny~\cite{GraphAny}, which combines predictors over multi-scale propagated features; TS-GNNs~\cite{TSGNN}, which impose permutation equivariance and label invariance; G2T-FM~\cite{G2T-FM} and TFM4GAD~\cite{TFM4GAD}, which augment node features with graph-based information and then pass them directly to tabular foundation models~\cite{TabPFNv2}; NodePFN~\cite{NodePFN} and GraphPFN~\cite{GraphPFN}, which extend prior-data fitted networks~\cite{PFNPosition} to node property prediction.
Despite their progress, existing GFMs mainly target supervised graph tasks. 
OFAG addresses a different and largely unexplored problem: foundation modeling for attributed graph clustering. 

\subsection{Prior-Data Fitted Networks}
\label{sec:rw-pfn}

Prior-data fitted networks (PFNs) provide a learning paradigm for amortizing Bayesian inference over tasks sampled from a prior~\citep{PFN, PFNPosition}. 
Consider a task-specific parameter $\phi\sim p(\phi)$ and a dataset $\mathcal{D}=\{(x_i,y_i)\}_{i=1}^{n}$ generated from $p(y\mid x,\phi)$. 
Given a context set $\mathcal{D}_{\mathrm{c}}$, Bayesian prediction for a query $x$ is given by the posterior predictive distribution
\begin{equation}
    p(y\mid x,\mathcal{D}_{\mathrm{c}})
    =
    \int p(y\mid x,\phi)\,
    p(\phi\mid \mathcal{D}_{\mathrm{c}})\,
    d\phi .
\end{equation}
Instead of computing this integral at test time, a PFN trains a neural network $q_\theta$ on synthetic datasets sampled from a prior $p(\mathcal{D})$:
\begin{equation}
    \mathcal{L}_{\mathrm{PFN}}(\theta)
    =
    \mathbb{E}_{\mathcal{D}_{\mathrm{c}}\cup\{(x,y)\}\sim p(\mathcal{D})}
    \left[
    -\log q_\theta(y\mid x,\mathcal{D}_{\mathrm{c}})
    \right].
\end{equation}
Minimizing this prior-data negative log-likelihood trains $q_\theta$ to approximate the posterior predictive distribution induced by the prior, enabling inference on new datasets in a single forward pass without gradient updates.
PFNs have shown strong empirical performance across several domains, such as tabular prediction~\cite{TabPFN, TabPFNv2}, time-series forecasting~\cite{ForecastPFN, CauKer}, and node property prediction~\cite{NodePFN, GraphPFN}.
OFAG follows the same high-level principle of prior fitting, but targets a different problem setting. 
It adapts the PFN philosophy from posterior predictive label inference to zero-shot clustering inference, learning from synthetic attributed graph priors how attributes and topology jointly reveal latent clusters.

\begin{figure*}[t]
    \centering
    \includegraphics[width=\textwidth]{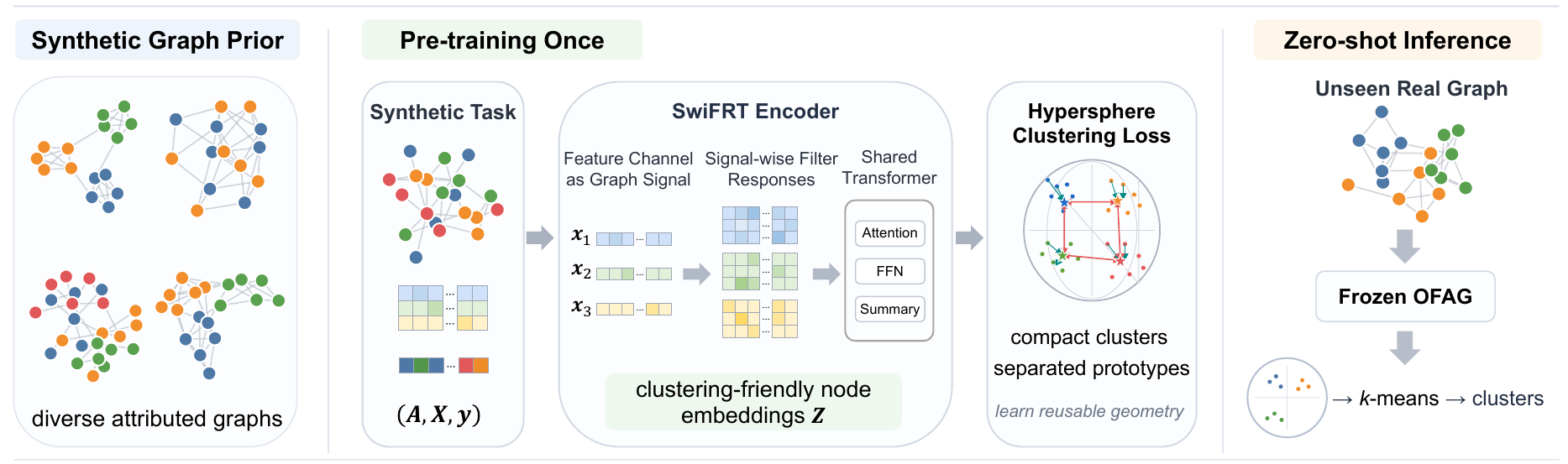}
    \caption{
        Overview of \texttt{OFAG}.
        The model is pre-trained on diverse synthetic attributed graphs using the SwiFRT encoder and a hyperspherical clustering objective.
        After pre-training, the frozen model can be directly applied to unseen real-world graphs for zero-shot clustering without graph-specific fine-tuning.
    }
    \label{fig:framework}
\end{figure*}

\section{Methodology}
\label{sec:method}
As illustrated in Figure~\ref{fig:framework}, \texttt{OFAG} learns a reusable \emph{representation geometry} from synthetic attributed graphs, enabling zero-shot clustering on arbitrary real graphs.
Three challenges drive the design: feature spaces across graphs are heterogeneous in both dimension and semantics; cluster identities carry no consistent meaning across graphs and $K$ varies; and node attributes and graph topology may provide complementary or conflicting evidence for clustering, depending on the graph.
We address these through three coupled components: the \emph{CAGM prior} (\S\ref{sec:synthetic_prior}), which generates synthetic tasks that cover the full diversity of real attributed graphs; \emph{SwiFRT} (\S\ref{sec:swifrt}), a dimension-agnostic encoder that operates on structural response patterns rather than raw feature values; and a \emph{hyperspherical clustering objective} (\S\ref{sec:vmf_loss}), which shapes the embedding space to be directly amenable to label-free clustering algorithms.

\subsection{Problem Setup}
\label{sec:method_problem}

Let $\mathcal{G}=(\mathcal{V},\mathcal{E},X)$ be an attributed graph with $N=|\mathcal{V}|$ nodes, adjacency matrix $A\in\{0,1\}^{N\times N}$, and feature matrix $X\in\mathbb{R}^{N\times F}$.
Attributed graph clustering aims to assign each node to one of $K$ latent clusters by jointly using $A$ and $X$.

Since cluster identities are arbitrary permutations of group indices and $K$ varies across graphs, \texttt{OFAG} does not learn to output fixed class labels.
Instead, it learns a \emph{representation geometry}: a mapping from $(A,X)$ to node embeddings $\{z_i\}$ in which same-cluster nodes are geometrically close and different-cluster nodes are separable, a structure that any label-free algorithm (e.g., $k$-means) can exploit at inference without supervision.

To acquire this geometry from synthetic supervision, let $\mathcal{D}=(A,X,y)$ denote one synthetic task, where $y=(y_1,\ldots,y_N)\in\{1,\ldots,K\}^N$ are latent cluster assignments available only during pre-training.
Given a prior $p(\mathcal{D})$, we optimize
\begin{equation}
    \theta^\ast = \arg\min_{\theta}\,\mathbb{E}_{(A,X,y)\sim p(\mathcal{D})}\bigl[\mathcal{L}\bigl(y,\,f_\theta(A,X)\bigr)\bigr].
    \label{eq:ofag_prior_objective}
\end{equation}
At inference, $y$ is unavailable; the frozen $f_{\theta^\ast}$ maps $(A,X)$ to embeddings clustered by a standard algorithm.
The design of $\mathcal{L}$ and $p(\mathcal{D})$, detailed in the following subsections, ensures that minimizing this expectation teaches a \emph{reusable} geometry rather than a dataset-specific solution.

\subsection{Synthetic Attributed Graph Prior}
\label{sec:synthetic_prior}

One-for-all generalization depends entirely on the breadth of $p(\mathcal{D})$: a prior that is too narrow teaches shortcuts that fail on out-of-distribution graphs.
We introduce the \emph{Coupled Attributed Graph Mixture} (CAGM) prior, a latent-variable generative model in which latent cluster memberships are the shared cause of both node attributes and graph edges, framing \texttt{OFAG} as learning to \emph{invert} a broad family of attributed graph generation processes.

\vpara{Generative structure.}
Each synthetic task samples $N$, $F$, and $K$ from broad marginals; cluster proportions $\pi\sim\mathrm{Dirichlet}(\eta\mathbf{1}_K)$ cover balanced and heavily imbalanced cluster sizes.
Node attributes are generated from cluster-conditioned latents $h_i^X\mid y_i=k\sim t_\nu(\mu_k,\Sigma_k)$, then decoded by a task-specific random transformation $T$ into continuous, categorical, count-valued, or bag-of-words features---covering the full diversity of real-world node attribute types.
Edges are drawn as $A_{ij}\sim\mathrm{Bernoulli}(\sigma(s_{ij}))$, where the log-odds
$s_{ij} = \rho + \alpha_\theta(\theta_i+\theta_j) + \lambda_B B_{y_i,y_j} + \lambda_H\kappa(h_i^A,h_j^A) + \lambda_X\mathrm{sim}(x_i,x_j)$
combines cluster affinity (block matrix $B$), degree propensities ($\theta_i$), latent structural proximity, and feature similarity.
The block affinity matrix $B$ is constructed from a continuously sampled edge homophily level $h\!\sim\!\mathrm{Uniform}(0.1,0.9)$, interpolating smoothly between homophilous and heterophilous structural regimes.
Full sampling details are in Appendix~\ref{app:prior_details}.

\vpara{Design rationale.}
The diversity of structural regimes is deliberate.
If the prior only generated homophilous graphs, \texttt{OFAG} could reduce to a feature-smoothing heuristic; if attributes were always well-separated, it could ignore topology altogether.
By independently randomizing attribute separability, edge--cluster coupling, and structural regime across tasks, CAGM compels the model to learn an adaptive strategy that weighs attribute and structural evidence according to their relative reliability in each graph.

The following theorem justifies why training on CAGM generalizes to real graphs; the proof is in Appendix~\ref{app:thm_synthetic_real}.

\begin{theorem}[Synthetic-to-real generalization]
\label{thm:synthetic_real}
Let $p_{\mathrm{real}}(\mathcal{G})$ be a distribution over attributed graphs satisfying:
\emph{(i)} the conditional feature distribution $p(X\mid y)$, given latent cluster labels $y$, is a finite mixture of continuous distributions with $y$-dependent parameters; and
\emph{(ii)} the edge distribution $p(A\mid y,X)$ follows a pairwise probability model whose log-odds may depend arbitrarily on cluster membership and node features.
Then there exist a graph encoder $\Phi$ and a CAGM prior $p_{\mathrm{syn}}(\mathcal{G})$ such that $\Phi_\# p_{\mathrm{syn}}(\mathcal{G})$ approximates $p_{\mathrm{real}}(\mathcal{G})$ with arbitrarily small total variation distance.
\end{theorem}

Conditions (i) and (ii) are mild: they assert that both attributes and edges are generated through cluster-dependent mechanisms, without restricting to any particular distribution family or connectivity pattern.

\subsection{Signal-wise Filter Response Transformer}
\label{sec:swifrt}

A dimension-specific feature transformation $XW$ with $W\in\mathbb{R}^{F\times d}$ is incompatible with zero-shot transfer: it binds the model to a fixed feature dimension $F$ and treats feature channels as semantically aligned across graphs, an assumption that fails when graphs come from different domains.

\texttt{OFAG} resolves this with the \emph{Signal-wise Filter Response Transformer} (SwiFRT), which changes what is encoded.
Inspired by graph signal processing~\citep{GSP}, SwiFRT treats each feature channel $X_{:,f}$ as a scalar signal on the graph and learns from its \emph{structural response} rather than its raw values.
The key observation is that the structural response is transferable across feature spaces: a signal that smooths uniformly under diffusion behaves differently from one that oscillates or amplifies with the local topology, regardless of the feature's semantic content.
SwiFRT therefore learns a single encoder operating in this response space, shared across all feature channels and all graphs.

\vpara{Filter response profiles.}
Let $\hat{A}=\tilde{D}^{-1/2}(A+I)\tilde{D}^{-1/2}$ be the symmetrically normalized adjacency matrix, where $\tilde{D}$ is the degree matrix of $A+I$~\cite{GCN,SSGC}.
For each node $i$ and feature channel $f$, the \emph{filter response profile} collects the signal value at each diffusion order $\ell=0,\ldots,L$:
\begin{equation}
    r_{i,f} = \bigl[(\hat{A}^0 X)_{i,f},\;(\hat{A}^1 X)_{i,f},\;\ldots,\;(\hat{A}^L X)_{i,f}\bigr]
    \in \mathbb{R}^{L+1}.
    \label{eq:swifrt_profile}
\end{equation}
This profile captures whether the signal at $(i,f)$ is locally stable, smoothed by neighbors, or progressively amplified or attenuated across multi-hop neighborhoods.

\vpara{Signal-wise transformer encoding.}
A shared encoder $g_\theta$, applied identically to every $(i,f)$ pair, maps each profile $r_{i,f}$ to a re-encoded scalar $Z_{i,f}=g_\theta(r_{i,f})$.
We implement $g_\theta$ as a Transformer operating along the filter-order dimension.
Each scalar response $r_{i,f,\ell}$ is first embedded as a $d$-dimensional token:
\begin{equation}
    e_{i,f,\ell} = \phi_{\mathrm{val}}(r_{i,f,\ell}) + p_\ell, \qquad \ell=0,\ldots,L,
    \label{eq:swifrt_token_embedding}
\end{equation}
where $\phi_{\mathrm{val}}:\mathbb{R}\to\mathbb{R}^d$ embeds the response value and $p_\ell\in\mathbb{R}^d$ is a learnable filter-order positional embedding encoding which diffusion step was applied.
Prepending a learnable summary token $c_{\mathrm{cls}}$, we run the sequence $[c_{\mathrm{cls}},e_{i,f,0},\ldots,e_{i,f,L}]$ through the Transformer~\cite{Transformer} and read off the summary-token output:
\begin{equation}
    h_{i,f} = \mathrm{Transformer}([c_{\mathrm{cls}},e_{i,f,0},\ldots,e_{i,f,L}])_{[0]} \in \mathbb{R}^d,
    Z_{i,f} = \phi_{\mathrm{out}}(h_{i,f}),
    \label{eq:swifrt_output}
\end{equation}
where $\phi_{\mathrm{out}}:\mathbb{R}^d\to\mathbb{R}$ is a shared scalar decoder.
The Transformer performs self-attention only among filter-response tokens of the same node-feature pair $(i,f)$. 
Thus, it learns interactions among different filter orders, such as whether a feature value is stable under diffusion, rapidly smoothed by neighbors, amplified at higher orders, or inconsistent with its surrounding graph structure.
These patterns characterize the signal's structural role independently of its semantic content.

The node representation is the row-normalized re-encoded feature vector:
\begin{equation}
    z_i = Z_{i,:}\,/\,\|Z_{i,:}\|_2 \in S^{F-1}.
    \label{eq:swifrt_norm}
\end{equation}
This construction gives SwiFRT two desirable properties for one-for-all clustering.
First, it is structure-aware: although the Transformer does not directly attend across nodes, each response profile already contains information aggregated by graph filters, so the scalar re-encoding reflects local and higher-order graph structure.
Second, it is dimension-agnostic: because $g_\theta$ is applied independently to each feature channel, the same model can handle arbitrary input feature dimensions.

The dimension-agnosticism and invariance to feature-channel ordering are formalized below; proofs are in Appendix~\ref{app:prop_swifrt}.

\begin{proposition}[Dimension-agnosticism and feature-permutation equivariance of SwiFRT]
\label{prop:swifrt}
Let $\sigma$ be any permutation of feature indices $\{1,\ldots,F\}$, and let $X_\sigma$ denote $X$ with the $f$-th column equal to the $\sigma(f)$-th column of $X$.
SwiFRT satisfies:
\emph{(i)} the number of trainable parameters is independent of the input feature dimension $F$; and
\emph{(ii)} $\operatorname{SwiFRT}_\theta(X_\sigma,A)=\operatorname{SwiFRT}_\theta(X,A)_\sigma$,
where the right-hand side denotes reordering the columns of $\operatorname{SwiFRT}_\theta(X,A)$ by $\sigma$.
\end{proposition}

\subsection{Hyperspherical Clustering Objective}
\label{sec:vmf_loss}

We train \texttt{OFAG} with an objective derived from the von Mises--Fisher (vMF) distribution~\cite{vMF}, which is the canonical probability distribution for directional data on the unit hypersphere.
Prior work has observed that clustering often becomes more stable and effective when performed on hyperspherical representations, where angular similarity rather than vector magnitude determines the grouping structure~\citep{SwAV, DGCBench}.
This is suitable for our setting because different graphs may have incompatible feature scales and dimensions.

\vpara{Prototypes.}
Given normalized embeddings $\{z_i\}$ from SwiFRT for a synthetic graph with labels $y_i\in\{1,\ldots,K\}$, we compute an \emph{episode-specific} prototype for each cluster $k$ on the fly:
\begin{equation}
    \hat{\mu}_k = \frac{\sum_{i:\,y_i=k} z_i}{\bigl\|\sum_{i:\,y_i=k} z_i\bigr\|_2 + \varepsilon} \in S^{F-1}.
    \label{eq:prototype}
\end{equation}
Computing prototypes on-the-fly from the current embeddings and known synthetic labels avoids global prototype parameters, accommodates variable $K$ across tasks, and introduces no fixed cluster-identity structure.

\vpara{Compactness loss.}
Under a vMF mixture with shared concentration $\kappa=1/\tau$, the posterior assignment probability is
\begin{equation}
    q_\theta(y_i=k\mid z_i)
    = \frac{\exp(\hat{\mu}_k^\top z_i/\tau)}{\sum_{k'=1}^K \exp(\hat{\mu}_{k'}^\top z_i/\tau)},
    \label{eq:vmf_posterior}
\end{equation}
and maximizing the likelihood of the synthetic assignments gives the compactness loss
\begin{equation}
    \mathcal{L}_{\mathrm{comp}}
    = -\frac{1}{N}\sum_{i=1}^N \log\frac{\exp(\hat{\mu}_{y_i}^\top z_i/\tau)}{\sum_{k=1}^K \exp(\hat{\mu}_k^\top z_i/\tau)}.
    \label{eq:compactness_loss}
\end{equation}
This term encourages each representation to align with its cluster prototype while moving away from other prototypes.

\begin{proposition}[Compactness loss as vMF MLE]
\label{prop:vmf_mle}
Minimizing $\mathcal{L}_{\mathrm{comp}}$ with respect to $\theta$ is equivalent to maximum likelihood estimation of the prototype directions $\{\hat{\mu}_k\}_{k=1}^K$ of a $K$-component vMF mixture with concentration $\kappa=1/\tau$.
\end{proposition}
\begin{proof}[Proof sketch]
The vMF normalization constant $Z_F(\kappa)$ does not depend on the mean direction $\mu_k$, so it cancels identically in Eq.~\eqref{eq:vmf_posterior}, yielding a posterior equivalent to the complete-data log-likelihood of a vMF mixture.
The full derivation is in Appendix~\ref{app:prop_vmf}.
\end{proof}

\vpara{Dispersion loss.}
The compactness loss alone does not prevent prototypes from collapsing into a small region of the hypersphere.
We therefore add a prototype dispersion term:
\begin{equation}
    \mathcal{L}_{\mathrm{disp}}
    = \frac{1}{K}\sum_{k=1}^K \log\frac{1}{K-1}\sum_{j\neq k} \exp\bigl(\hat{\mu}_k^\top \hat{\mu}_j/\tau\bigr),
    \label{eq:dispersion_loss}
\end{equation}
whose minimization reduces pairwise cosine similarities between different prototypes and encourages the clusters to spread across the hypersphere.

\begin{remark}[Equiangular tight frame optimum]
\label{rem:etf}
$\mathcal{L}_{\mathrm{disp}}$ is a log-sum-exp of pairwise cosine similarities between prototypes.
When $F\geq K-1$, its minimizer approaches an \emph{equiangular tight frame} (ETF): the unique configuration in which all pairwise cosine similarities equal $-1/(K-1)$ and prototypes are as uniformly spread across $S^{F-1}$ as possible.
This geometry maximizes inter-cluster angular separation and is most favorable for distinguishing unseen nodes between clusters at inference.
\end{remark}

\begin{remark}[Complementarity of $\mathcal{L}_{\mathrm{comp}}$ and $\mathcal{L}_{\mathrm{disp}}$]
\label{rem:complementarity}
The two terms regularize complementary aspects of the mixture geometry: $\mathcal{L}_{\mathrm{comp}}$ concentrates each node near its prototype while repelling it from all other prototypes; $\mathcal{L}_{\mathrm{disp}}$ spreads the prototypes themselves across the hypersphere.
Together they enforce that nodes concentrate \emph{within} clusters and that clusters are separated \emph{from each other}.
\end{remark}

\vpara{Training objective.}
The overall prior-data objective is
\begin{equation}
    \mathcal{L}(\theta)
    = \mathbb{E}_{(A,X,y)\sim p(\mathcal{D})}
    \bigl[\mathcal{L}_{\mathrm{comp}}(\theta;A,X,y) + \lambda\,\mathcal{L}_{\mathrm{disp}}(\theta;A,X,y)\bigr],
    \label{eq:expected_full_loss}
\end{equation}
where $\lambda>0$ balances the two terms.
In practice, this expectation is approximated by Monte Carlo: at each training step we sample a batch of synthetic graphs, encode them with SwiFRT, compute on-the-fly prototypes from the synthetic labels, and back-propagate through the combined loss.

\begin{table*}[t]
\centering
\caption{Clustering performance in NMI and ACC on homophilous and heterophilous attributed graphs. The \bestcap{}, \secondcap{}, and \thirdcap{} results in each column are highlighted. For Avg. Rank, lower is better; for all other columns, higher is better.}
\label{tab:nmi_acc}
\scriptsize
\setlength{\tabcolsep}{2.2pt}
\resizebox{\textwidth}{!}{%
\begin{tabular}{llcccccccccccc}
\toprule
\multirow{2}{*}{Model} & \multirow{2}{*}{Metric} & \multicolumn{5}{c}{Homophilous Graphs} & \multicolumn{5}{c}{Heterophilous Graphs} & \multirow{2}{*}{Mean} & \multirow{2}{*}{Avg. Rank} \\
\cmidrule(lr){3-7} \cmidrule(lr){8-12}
& & Cora & Photo & ArXiv & Reddit & Products & HM & Flickr & MAG & Pokec & WebTopic & & \\
\midrule
\multirow{2}{*}{KMeans} & NMI & \res{13.89}{4.46} & \res{31.70}{0.49} & \res{22.59}{0.03} & \res{11.08}{0.14} & \res{29.24}{0.11} & \res{10.17}{0.11} & \res{1.21}{0.07} & \res{28.32}{0.06} & \res{1.41}{0.00} & \res{2.45}{0.05} & 15.21 & 16.30 \\
    & ACC & \res{33.09}{4.00} & \res{42.54}{2.54} & \res{16.95}{0.38} & \res{8.59}{0.06} & \res{21.23}{0.74} & \res{14.12}{0.25} & \res{26.02}{1.73} & \res{7.09}{0.03} & \res{1.91}{0.01} & \res{9.63}{0.34} & 18.12 & 15.90 \\
\addlinespace[1pt]
\multirow{2}{*}{Node2Vec} & NMI & \res{44.95}{1.49} & \res{66.02}{1.05} & \res{38.93}{0.15} & \res{79.23}{0.38} & \res{50.97}{0.20} & \res{7.60}{0.13} & \res{5.59}{0.00} & \res{37.73}{0.03} & \second{\res{31.92}{0.02}} & \res{5.80}{0.09} & 36.87 & 9.80 \\
    & ACC & \res{60.52}{2.81} & \res{73.08}{1.06} & \res{25.97}{0.57} & \res{70.67}{1.68} & \res{36.59}{0.74} & \res{13.28}{0.41} & \res{24.44}{0.01} & \res{10.86}{0.09} & \second{\res{16.62}{0.19}} & \res{10.71}{0.12} & 34.27 & 10.60 \\
\addlinespace[1pt]
\midrule
\multirow{2}{*}{SSGC} & NMI & \res{51.85}{1.05} & \res{70.75}{1.39} & \res{46.12}{0.16} & \res{51.61}{0.36} & \res{52.04}{0.17} & \res{12.15}{0.11} & \res{4.48}{0.19} & \third{\res{41.52}{0.04}} & \res{4.33}{0.01} & \res{3.92}{0.11} & 33.88 & 7.60 \\
    & ACC & \res{65.38}{1.47} & \res{76.34}{4.04} & \second{\res{38.58}{0.56}} & \res{40.00}{1.79} & \res{37.02}{0.70} & \res{15.57}{0.07} & \res{25.95}{0.42} & \res{11.89}{0.10} & \res{2.93}{0.04} & \res{11.34}{0.18} & 32.50 & 9.10 \\
\addlinespace[1pt]
\multirow{2}{*}{SAGSC} & NMI & \res{44.35}{0.06} & \res{58.40}{0.02} & \res{43.27}{0.19} & \third{\res{80.02}{0.38}} & \res{51.78}{0.24} & \third{\res{12.47}{0.13}} & \res{6.86}{0.00} & \res{40.15}{0.03} & \best{\res{38.33}{0.03}} & \third{\res{9.11}{0.26}} & \second{38.47} & 6.90 \\
    & ACC & \res{63.45}{0.03} & \res{67.91}{0.01} & \res{30.87}{0.83} & \third{\res{72.79}{1.57}} & \res{38.11}{0.69} & \third{\res{17.83}{0.11}} & \res{26.72}{0.01} & \res{11.90}{0.16} & \best{\res{19.86}{0.23}} & \res{13.66}{0.53} & \third{36.31} & 6.95 \\
\addlinespace[1pt]
\multirow{2}{*}{MS2CAG} & NMI & \res{53.64}{0.64} & \second{\res{72.49}{0.79}} & \res{44.20}{0.18} & \res{72.85}{0.53} & \res{50.71}{0.31} & \res{9.75}{0.36} & \res{7.32}{0.02} & \res{40.39}{0.07} & \res{2.77}{0.05} & \res{7.77}{0.08} & 36.19 & 7.70 \\
    & ACC & \res{70.29}{0.66} & \second{\res{79.13}{0.05}} & \res{35.85}{0.21} & \res{60.40}{2.40} & \res{35.41}{0.62} & \res{15.06}{0.28} & \res{28.29}{0.06} & \res{11.70}{0.10} & \res{2.27}{0.09} & \res{12.38}{0.12} & 35.08 & 8.10 \\
\addlinespace[1pt]
\midrule
\multirow{2}{*}{DAEGC} & NMI & \res{46.88}{1.71} & \res{63.60}{0.02} & \res{39.81}{0.29} & \res{40.84}{0.76} & \res{9.18}{3.05} & \res{11.46}{0.09} & \res{4.23}{0.03} & \res{29.00}{2.17} & \res{3.85}{0.21} & \res{4.79}{0.29} & 25.36 & 13.40 \\
    & ACC & \res{64.06}{3.18} & \res{77.78}{0.01} & \res{28.68}{0.64} & \res{27.57}{0.33} & \res{14.83}{2.09} & \res{17.37}{0.31} & \res{28.51}{0.35} & \res{10.83}{0.51} & \res{3.99}{0.03} & \third{\res{16.42}{0.61}} & 29.00 & 9.20 \\
\addlinespace[1pt]
\multirow{2}{*}{DinkNet} & NMI & \res{55.56}{0.18} & \res{64.74}{0.06} & \res{37.24}{0.57} & \res{54.95}{0.24} & \res{38.43}{0.18} & \res{11.26}{0.06} & \res{6.67}{0.05} & \res{37.12}{0.05} & \res{3.91}{0.01} & \res{5.47}{0.04} & 31.54 & 11.00 \\
    & ACC & \res{72.49}{0.50} & \res{74.03}{0.25} & \res{25.47}{0.88} & \res{38.51}{0.82} & \res{26.78}{0.37} & \res{15.79}{0.28} & \second{\res{29.94}{0.38}} & \res{10.14}{0.08} & \res{2.73}{0.04} & \res{11.52}{0.55} & 30.74 & 10.50 \\
\addlinespace[1pt]
\multirow{2}{*}{MinCut} & NMI & \res{40.80}{1.83} & \res{62.34}{2.39} & \res{39.00}{1.02} & \res{48.85}{3.58} & \res{35.80}{0.71} & \res{7.42}{0.37} & \res{7.59}{0.19} & \res{38.21}{0.12} & \res{5.78}{0.18} & \res{6.07}{0.52} & 29.19 & 12.10 \\
    & ACC & \res{52.95}{6.65} & \res{69.46}{2.22} & \res{27.88}{1.20} & \res{34.09}{2.08} & \res{24.88}{0.74} & \res{13.87}{0.43} & \res{27.64}{1.87} & \res{11.60}{0.32} & \res{3.17}{0.09} & \res{11.67}{0.95} & 27.72 & 12.60 \\
\addlinespace[1pt]
\multirow{2}{*}{DMoN} & NMI & \res{43.84}{2.29} & \res{62.74}{2.67} & \res{38.77}{0.28} & \res{50.66}{0.38} & \res{34.80}{0.85} & \res{7.59}{0.92} & \second{\res{7.89}{0.09}} & \res{38.30}{0.12} & \res{8.42}{0.08} & \res{7.49}{0.37} & 30.05 & 10.95 \\
    & ACC & \res{52.99}{4.70} & \res{72.68}{1.87} & \res{25.89}{1.12} & \res{39.29}{1.29} & \res{25.40}{1.37} & \res{12.32}{0.58} & \res{28.64}{0.69} & \res{10.81}{0.08} & \res{4.40}{0.07} & \res{12.70}{0.73} & 28.51 & 11.30 \\
\addlinespace[1pt]
\multirow{2}{*}{Neuromap} & NMI & \res{46.98}{2.61} & \res{61.84}{2.34} & \res{40.70}{0.80} & \res{47.81}{0.28} & \res{34.80}{1.46} & \res{7.44}{0.17} & \res{7.77}{0.57} & \res{39.14}{0.20} & \res{8.53}{0.25} & \res{7.19}{0.51} & 30.22 & 10.85 \\
    & ACC & \res{56.96}{3.98} & \res{71.89}{4.19} & \res{33.31}{1.07} & \res{32.64}{0.20} & \res{31.87}{1.93} & \res{12.73}{0.25} & \res{28.86}{2.39} & \best{\res{15.95}{0.32}} & \third{\res{5.12}{0.21}} & \second{\res{16.82}{1.69}} & 30.61 & 8.40 \\
\addlinespace[1pt]
\midrule
\multirow{2}{*}{GAE} & NMI & \res{50.09}{0.07} & \res{61.36}{0.14} & \res{40.86}{0.15} & \res{45.90}{0.27} & \res{42.55}{0.09} & \second{\res{13.58}{0.12}} & \res{4.07}{0.05} & \res{39.30}{0.03} & \res{3.79}{0.04} & \res{3.89}{0.14} & 30.54 & 11.70 \\
    & ACC & \res{67.53}{0.13} & \res{71.54}{0.08} & \res{24.39}{0.31} & \res{34.56}{0.81} & \res{27.70}{0.43} & \res{16.58}{0.32} & \best{\res{38.42}{1.07}} & \res{11.21}{0.07} & \res{2.49}{0.02} & \res{11.86}{0.06} & 30.63 & 10.40 \\
\addlinespace[1pt]
\multirow{2}{*}{DGI} & NMI & \res{56.22}{0.87} & \res{67.77}{0.54} & \res{42.29}{0.08} & \res{72.93}{0.14} & \res{41.28}{0.13} & \res{11.67}{0.06} & \res{6.88}{0.01} & \res{39.39}{0.03} & \res{4.20}{0.01} & \res{5.91}{0.02} & 34.85 & 8.00 \\
    & ACC & \res{72.22}{2.92} & \res{77.57}{0.28} & \res{30.49}{0.64} & \res{62.52}{0.34} & \res{28.99}{0.52} & \res{15.78}{0.32} & \res{24.92}{0.17} & \res{10.60}{0.11} & \res{2.23}{0.02} & \res{11.69}{0.10} & 33.70 & 10.15 \\
\addlinespace[1pt]
\multirow{2}{*}{CCASSG} & NMI & \third{\res{58.74}{0.87}} & \res{64.54}{3.23} & \res{44.69}{0.04} & \res{49.63}{0.07} & \res{50.89}{0.31} & \res{11.93}{0.15} & \res{4.67}{0.01} & \res{40.40}{0.02} & \res{1.35}{0.01} & \res{3.79}{0.18} & 33.06 & 9.40 \\
    & ACC & \res{73.35}{1.92} & \res{70.90}{2.27} & \res{30.68}{0.27} & \res{37.91}{0.49} & \res{37.99}{0.42} & \res{15.78}{0.14} & \res{23.03}{0.35} & \res{11.90}{0.05} & \res{1.95}{0.01} & \res{11.65}{0.91} & 31.51 & 10.60 \\
\addlinespace[1pt]
\multirow{2}{*}{S3GC} & NMI & \res{55.45}{1.12} & \res{68.27}{2.22} & \second{\res{47.11}{0.12}} & \second{\res{83.45}{0.39}} & \second{\res{53.43}{0.16}} & \res{11.57}{0.08} & \third{\res{7.84}{0.00}} & \res{39.80}{0.03} & \res{6.04}{0.00} & \res{7.83}{0.02} & \third{38.08} & \third{5.10} \\
    & ACC & \res{70.33}{1.98} & \res{75.92}{3.26} & \res{35.57}{0.60} & \second{\res{81.86}{1.59}} & \second{\res{40.45}{0.62}} & \res{16.49}{0.22} & \res{28.48}{0.00} & \res{12.15}{0.05} & \res{2.97}{0.03} & \res{12.75}{0.07} & \second{37.70} & 6.00 \\
\addlinespace[1pt]
\multirow{2}{*}{MAGI} & NMI & \second{\res{58.94}{0.41}} & \res{68.65}{0.16} & \third{\res{46.53}{0.16}} & \res{72.53}{0.19} & \res{44.58}{0.14} & \res{11.24}{0.46} & \res{6.31}{0.16} & \res{41.34}{0.02} & \res{8.76}{0.02} & \second{\res{9.27}{0.20}} & 36.82 & 5.75 \\
    & ACC & \second{\res{74.22}{0.57}} & \res{76.66}{0.18} & \third{\res{37.32}{0.74}} & \res{64.78}{1.38} & \res{33.15}{0.34} & \best{\res{18.49}{0.83}} & \res{24.28}{0.46} & \third{\res{13.25}{0.07}} & \res{3.90}{0.03} & \res{14.37}{0.76} & 36.04 & \third{5.60} \\
\addlinespace[1pt]
\multirow{2}{*}{NS4GC} & NMI & \best{\res{59.40}{0.48}} & \best{\res{72.62}{0.79}} & \best{\res{48.39}{0.24}} & \res{56.71}{0.05} & \best{\res{54.63}{0.14}} & \best{\res{15.28}{0.17}} & \res{6.19}{0.01} & \second{\res{41.64}{0.01}} & \res{7.19}{0.04} & \best{\res{10.05}{0.04}} & 37.21 & \second{3.40} \\
    & ACC & \best{\res{74.92}{0.59}} & \best{\res{79.21}{0.03}} & \best{\res{39.01}{1.07}} & \res{43.80}{0.52} & \best{\res{41.83}{0.40}} & \second{\res{18.40}{0.42}} & \res{24.27}{0.06} & \second{\res{13.28}{0.09}} & \res{3.19}{0.01} & \res{15.69}{0.14} & 35.36 & \second{4.40} \\
\addlinespace[1pt]
\midrule
\multirow{2}{*}{\texttt{OFAG}} & NMI & \res{57.06}{0.00} & \third{\res{71.45}{0.00}} & \res{45.08}{0.19} & \best{\res{85.94}{0.00}} & \third{\res{52.77}{0.00}} & \res{11.61}{0.19} & \best{\res{9.27}{0.00}} & \best{\res{41.78}{0.02}} & \third{\res{10.63}{0.05}} & \second{\res{9.27}{0.00}} & \best{39.49} & \best{3.05} \\
    & ACC & \third{\res{74.04}{0.00}} & \third{\res{79.05}{0.00}} & \res{34.75}{0.13} & \best{\res{88.68}{0.00}} & \third{\res{38.51}{0.04}} & \res{17.38}{0.24} & \third{\res{29.02}{0.00}} & \res{12.57}{0.05} & \res{4.78}{0.02} & \best{\res{17.15}{0.00}} & \best{39.59} & \best{3.20} \\
\bottomrule
\end{tabular}%
}
\end{table*}

\subsection{Connection to Prior-Data Fitted Networks}
\label{sec:pfn_connection}

The prior-fitted clustering objective is not merely a convenient training recipe; it has a precise Bayesian interpretation that places \texttt{OFAG} within the PFN framework~\citep{PFN,PFNPosition}.

\begin{proposition}[\texttt{OFAG} as a Prior-Data Fitted Network]
\label{prop:ofag_pfn}
Minimizing $\mathbb{E}_{(A,X,y)\sim p(\mathcal{D})}[\mathcal{L}_{\mathrm{comp}}(\theta;A,X,y)]$ is equivalent to minimizing the Prior-Data Negative Log-Likelihood of a Bayesian model with prior $p(\mathcal{D})$ and a vMF mixture likelihood, with empty context $\mathcal{D}_{\mathrm{c}}=\emptyset$.
Consequently, a sufficiently expressive \texttt{OFAG} trained to convergence approximates the posterior predictive distribution induced by $p(\mathcal{D})$.
\end{proposition}
\begin{proof}[Proof sketch]
With $\mathcal{D}_{\mathrm{c}}=\emptyset$ and all nodes as query points, the Prior-Data NLL equals
$\mathbb{E}\bigl[-\frac{1}{N}\sum_i\log q_\theta(y_i\mid z_i)\bigr]=\mathbb{E}[\mathcal{L}_{\mathrm{comp}}]$,
where $q_\theta(y_i\mid z_i)$ is the vMF posterior from Proposition~\ref{prop:vmf_mle}.
By~\citep{PFN,PFNPosition}, minimizing the Prior-Data NLL is equivalent to minimizing the expected KL divergence between $q_\theta$ and the true posterior predictive distribution.
The full proof is in Appendix~\ref{app:prop_pfn}.
\end{proof}

The dispersion term $\mathcal{L}_{\mathrm{disp}}$ is a geometric regularizer preventing prototype collapse; it does not affect the PFN interpretation of $\mathcal{L}_{\mathrm{comp}}$.
Together, Propositions~\ref{prop:vmf_mle} and~\ref{prop:ofag_pfn} establish that \texttt{OFAG}'s training objective is grounded in both probabilistic generative modeling (vMF mixture MLE) and Bayesian inference theory (Prior-Data NLL minimization).

\subsection{Zero-Shot Inference}
\label{sec:inference}

After pre-training, all parameters are frozen.
Given an unseen attributed graph $(A,X)$, \texttt{OFAG} produces node embeddings in a single forward pass,
\begin{equation}
    Z = \operatorname{SwiFRT}_{\theta^\ast}(X,A), \qquad z_i = Z_{i,:}/\|Z_{i,:}\|_2,
    \label{eq:inference}
\end{equation}
which are then clustered by a standard algorithm such as $k$-means.
No gradient updates, graph-specific training, or validation labels are required.
The entire burden of learning how to combine attribute and structural evidence across diverse feature types, structural regimes, and cluster configurations is absorbed by prior fitting.
At inference, \texttt{OFAG} follows a pure foundation-model workflow: pre-train once, freeze, and apply to any new graph.

\begin{table*}[t]
    \centering
    \caption{Wall-clock clustering time on benchmark datasets. Following PyAGC, all methods are evaluated on a single NVIDIA Tesla V100 GPU with 32GB memory. Cora and Photo are reported in seconds, while the remaining datasets and total time are reported in minutes. Lower is better. The $\dagger$ mark indicates that \texttt{OFAG} is run on CPU due to GPU out-of-memory.}
    \label{tab:efficiency}
    \scriptsize
    \setlength{\tabcolsep}{2.6pt}
    \resizebox{\textwidth}{!}{%
    \begin{tabular}{lrrrrrrrrrrrr}
    \toprule
    Method & Cora (s) & Photo (s) & HM (m) & Flickr (m) & ArXiv (m) & Reddit (m) & MAG (m) & Pokec (m) & Products (m) & WebTopic (m) & Total (m) & Mean NMI \\
    \midrule
    Node2Vec & 24.57 & 27.96 & 1.50 & 20.97 & 5.05 & 22.67 & 77.58 & 236.53 & 64.94 & 75.92 & 506.04 & 36.87 \\
    DAEGC    & 35.80 & 156.84 & 1.10 & 4.87 & 9.06 & 14.28 & 27.52 & 14.28 & 37.28 & 11.11 & 122.71 & 25.36 \\
    DinkNet  & 13.31 & 36.07 & 0.14 & 1.40 & 3.72 & 4.49 & 5.46 & 6.18 & 85.27 & 4.27 & 111.75 & 31.54 \\
    MinCut   & 5.51 & 6.78 & 5.11 & 0.34 & 1.00 & 6.73 & 16.01 & 24.15 & 32.55 & 8.37 & 94.46 & 29.19 \\
    DMoN     & 3.58 & 6.99 & 5.08 & 0.34 & 1.37 & 1.22 & 15.46 & 22.91 & 22.88 & 6.75 & \second{76.19} & 30.05 \\
    Neuromap & 4.73 & 12.89 & 5.17 & 0.47 & 1.62 & 6.73 & 14.56 & 23.07 & 23.00 & 4.97 & \third{79.88} & 30.22 \\
    GAE      & 32.03 & 178.93 & 0.81 & 2.56 & 12.77 & 30.58 & 26.54 & 10.16 & 26.12 & 6.18 & 119.24 & 30.54 \\
    DGI      & 15.50 & 11.14 & 0.54 & 1.12 & 2.76 & 1.93 & 10.54 & 10.37 & 50.32 & 6.16 & 84.18 & 34.85 \\
    CCASSG   & 0.46 & 27.92 & 0.18 & 0.09 & 3.04 & 1.21 & 31.14 & 117.43 & 45.16 & 122.30 & 321.02 & 33.06 \\
    S3GC     & 124.74 & 148.44 & 0.98 & 15.49 & 3.91 & 32.38 & 41.78 & 71.17 & 29.68 & 145.98 & 345.92 & \second{38.08} \\
    MAGI     & 51.25 & 18.68 & 41.84 & 43.31 & 28.80 & 56.98 & 86.99 & 176.43 & 375.99 & 115.36 & 926.87 & 36.82 \\
    NS4GC    & 9.77 & 64.36 & 4.47 & 0.85 & 22.94 & 12.84 & 41.04 & 55.88 & 37.24 & 30.75 & 207.25 & \third{37.21} \\
    \midrule
    \texttt{OFAG} & \textbf{0.40} & \textbf{1.65} & \textbf{0.05} & \textbf{0.05} & \textbf{0.12} & \textbf{0.42} & \textbf{0.96} & \textbf{1.11}$^\dagger$ & \textbf{6.21}$^\dagger$ & \textbf{3.47}$^\dagger$ & \best{12.43} & \best{39.49} \\
    \bottomrule
    \end{tabular}%
    }
    \end{table*}

\section{Experiments}
\label{sec:exp}

\subsection{Experimental Setup}
\label{sec:exp_setup}

\vpara{Datasets.}
We evaluate \texttt{OFAG} following the protocol of PyAGC~\citep{PyAGC} on ten benchmark datasets spanning citation, co-purchase, social, and web domains: Cora, Photo, HM, Flickr, ArXiv, Reddit, MAG, Pokec, Products, and WebTopic~\citep{Position, OGB, GraphLand}.
The benchmark covers diverse graph scales, feature modalities, and homophily levels, with 5 homophilous and 5 heterophilous graphs.
Detailed dataset statistics and descriptions are provided in Appendix~\ref{app:datasets}.

\vpara{Baselines.}
We compare against 16 methods spanning traditional clustering, deep decoupled clustering, and deep joint clustering: KMeans~\citep{KMeans}, Node2Vec~\citep{Node2Vec}, SSGC~\citep{SSGC}, SAGSC~\citep{SAGSC}, MS2CAG~\citep{MS2CAG}, GAE~\citep{GAE}, DGI~\citep{DGI}, CCASSG~\citep{CCASSG}, S3GC~\citep{S3GC}, NS4GC~\citep{NS4GC}, MAGI~\citep{MAGI}, DAEGC~\citep{DAEGC}, DinkNet~\citep{DinkNet}, MinCut~\citep{MinCut}, DMoN~\citep{DMoN}, and Neuromap~\citep{Neuromap}.
We follow the official PyAGC implementations and evaluation settings; details are in Appendix~\ref{app:baselines}.

\vpara{Metrics.}
We report NMI, ARI, ACC, and Macro-F1, averaged over 5 runs with different seeds, and wall-clock running time.
For baselines, running time covers on-graph learning plus clustering; for \texttt{OFAG}, it covers only zero-shot inference (frozen forward pass and $k$-means), excluding the one-time pre-training cost.

\vpara{Implementation.}
\texttt{OFAG} is pre-trained once on 80,000 synthetic graphs (batch size 8, 10,000 steps) using AdamW (lr $10^{-4}$, wd $10^{-4}$), taking $\approx1.5$ hours on a single NVIDIA A800 40G GPU.
The SwiFRT hidden dimension is 8, num\_heads=4, filter order $L=10$, temperature $\tau=1.0$, and dispersion weight $\lambda=0.1$.
A fixed set of 16 synthetic graphs is held out as a validation set to monitor pre-training; the checkpoint with the lowest validation loss is selected.
At inference time, all parameters are frozen: \texttt{OFAG} performs a single forward pass followed by $k$-means, with no graph-specific training, fine-tuning, or hyperparameter search on any benchmark dataset.
Our code is available at \url{https://github.com/Cloudy1225/OFAG}.
More implementation and evaluation details are in Appendix~\ref{app:implementation}.

\subsection{Performance Comparison}
\label{sec:performance}

\vpara{Overall performance.}
Tables~\ref{tab:nmi_acc} and~\ref{tab:ari_f1} report clustering performance across all ten datasets.
\texttt{OFAG} achieves the best mean score on all four metrics (NMI 39.49, ACC 39.59, ARI 26.94, F1 33.54) and the lowest average rank on every metric, demonstrating consistent superiority rather than gains concentrated on a few datasets.
The next-best methods on mean NMI and F1 are SAGSC (38.47, 32.36) and S3GC (37.70 ACC, 26.06 ARI), both of which require full per-graph optimization; \texttt{OFAG} surpasses them by a larger margin on mean scores than the gap between any two consecutive dataset-specific methods.
Crucially, this is achieved by a single frozen model performing zero-shot inference without any graph-specific training, fine-tuning, or hyperparameter search.

\vpara{Homophilous graphs.}
\texttt{OFAG} is strongest on Reddit, leading all methods by a wide margin (NMI 85.94 vs.\ S3GC 83.45; F1 81.17 vs.\ SAGSC 70.43); the $+10.7$ F1 gap suggests that the CAGM prior's diverse cluster-proportion sampling is especially beneficial for Reddit's imbalanced community structure.
On Cora, Photo, and ArXiv, \texttt{OFAG} ranks second or third, trailing specialized baselines by 2--4 NMI points, a modest gap attributable to dataset-specific optimization advantages that zero-shot inference cannot replicate.

\vpara{Heterophilous graphs.}
\texttt{OFAG} shows clear advantages where smoothing-based methods degrade.
It ranks first on Flickr (NMI 9.27, F1 23.89), MAG (NMI 41.78, F1 10.28), and WebTopic (ACC 17.15), consistent with its pre-training on a comprehensive homophily spectrum that prevents over-reliance on edge propagation. 
On HM and Pokec, \texttt{OFAG} does not rank in the top three on most metrics.
This may be because both datasets are dense tabular graphs with many fine-grained clusters and attribute-structure correlations that diverge from typical prior conditions, where per-graph hyperparameter tuning provides a non-trivial advantage.

\subsection{Efficiency Comparison}
\label{sec:efficiency}

Table~\ref{tab:efficiency} jointly reports wall-clock running times and Mean NMI, enabling a direct quality-vs.-cost comparison.
\texttt{OFAG} is simultaneously the fastest and the most accurate method: it achieves the best Mean NMI (39.49) while completing all ten datasets in only 12.43 minutes total.
The next fastest baseline is DMoN (76.19 minutes, $6.1\times$ slower) with a Mean NMI of only 30.05; the second-best on NMI is S3GC (38.08) at 345.92 minutes, $27.8\times$ slower than \texttt{OFAG}.
This reveals a sharp Pareto gap: no dataset-specific method comes close to the quality of \texttt{OFAG} at anywhere near its speed.
The most extreme case is MAGI (926.87 minutes, $74.6\times$ slower) which achieves only 36.82 Mean NMI, worse than \texttt{OFAG} by 2.67 points.
The efficiency advantage is most pronounced on individual large graphs: on ArXiv, \texttt{OFAG} finishes in 0.12 minutes while NS4GC, the best-performing baseline on that dataset, requires 22.94 minutes ($191\times$ slower) yet scores lower on average NMI.
This gap widens further when accounting for hyperparameter tuning: dataset-specific methods must be tuned on each target graph without labels, often requiring tens of repeated runs, pushing practical speedups to the thousand-fold level.
\texttt{OFAG} avoids this cost entirely: its inference consists of a single forward pass followed by $k$-means, with no per-graph adaptation of any kind.
For the three largest datasets (Pokec, Products, WebTopic), \texttt{OFAG} runs on CPU due to GPU memory constraints, yet still finishes faster than all GPU-accelerated baselines in total, further underscoring the practical scalability of prior-fitted inference.
A formal time complexity analysis of \texttt{OFAG} is provided in Appendix~\ref{app:complexity}.

\subsection{Ablation: Model Components}
\label{sec:ablation_components}

Table~\ref{tab:ablation_components} ablates five design choices grouped into three categories: synthetic prior, backbone encoder, and training objective.

\vpara{Prior.}
Including mixed attribute types (binary, categorical, count-valued) in CAGM yields consistent gains on both homophilous and heterophilous splits, confirming that broader coverage over real-world feature modalities improves zero-shot generalization.
Covering the full homophily spectrum ($h\!\in\![0.1,0.9]$) is similarly important: it equips the model with an adaptive aggregation strategy rather than a fixed smoothing bias, benefiting heterophilous graphs the most.
More prior ablations are provided in Appendix~\ref{app:ablation_homophily}.

\vpara{Backbone.}
The Transformer fusion in SwiFRT is essential for heterophilous graphs ($+2.71$ NMI over the MLP variant), where diffusion responses are non-monotone across orders and cross-order interactions must be explicitly modeled.
$\ell_2$ normalization contributes the largest single gain in the entire table: projecting node embeddings onto the unit hypersphere removes scale variation across graphs and feature dimensions, producing a geometry in which $k$-means cluster boundaries are well-calibrated.

\vpara{Objective.}
$\mathcal{L}_{\mathrm{disp}}$ provides a clear benefit, particularly on heterophilous benchmarks ($+2.01$ NMI).
As formalized in Remarks~\ref{rem:etf}--\ref{rem:complementarity}, it drives prototypes toward an equiangular tight frame, maximizing inter-cluster angular separation across the hypersphere.
This geometric regularization is most consequential when the compactness signal is weaker---precisely the regime of heterophilous graphs where attribute-structure agreement is limited and inter-prototype spacing must compensate.

\begin{table}[h]
\centering
\caption{Component ablation. Average NMI/ACC over the five homophilous (Homo), five heterophilous (Hetero), and all ten graphs (All).}
\label{tab:ablation_components}
\small
\setlength{\tabcolsep}{4pt}
\begin{tabular}{lcccccc}
\toprule
\multirow{2}{*}{Variant} & \multicolumn{2}{c}{Homo Avg} & \multicolumn{2}{c}{Hetero Avg} & \multicolumn{2}{c}{All Avg} \\
\cmidrule(lr){2-3}\cmidrule(lr){4-5}\cmidrule(lr){6-7}
    & NMI & ACC & NMI & ACC & NMI & ACC \\
\midrule
\texttt{OFAG} (full)                           & \textbf{62.46} & \textbf{63.01} & \textbf{16.51} & \textbf{16.18} & \textbf{39.49} & \textbf{39.59} \\
\midrule
Continuous features only                       & 62.02 & 62.35 & 16.12 & 15.83 & 39.07 & 39.09 \\
Homophilous graphs only & 61.46 & 59.24 & 15.65 & 15.22 & 38.56 & 37.23 \\
$g_\theta$: Transformer $\to$ MLP                  & 60.24 & 59.41 & 13.80 & 14.32 & 37.02 & 36.86 \\
w/o $\ell_2$ normalization                     & 54.72 & 48.70 & 14.10 & 15.27 & 34.41 & 31.98 \\
w/o dispersion  $\mathcal{L}_{\mathrm{disp}}$              & 61.40 & 59.54 & 14.50 & 14.54 & 37.95 & 37.04 \\
\bottomrule
\end{tabular}
\end{table}

\subsection{Pre-training Dynamics}
\label{sec:pretrain_dynamics}

Figure~\ref{fig:pretrain} shows the validation loss and NMI evaluated every 50 steps on a fixed set of 16 synthetic graphs throughout pre-training.
Both curves share a consistent two-phase pattern: a steep improvement in the first $\sim$1{,}000 steps, followed by a gradual plateau that persists for the remainder of training.
The monotonically increasing NMI directly reflects that the quality of the learned representations improves over training: the model continuously discovers new structural and attributive patterns in the synthetic prior, yielding progressively better cluster separability.
The validation loss decreases in tandem, confirming that the pre-training objective faithfully drives clustering quality rather than reducing a decoupled surrogate.
After the initial rapid gains, both metrics stabilize with only low-amplitude fluctuations, indicating that the model has reached a stable convergence point rather than continuing to overfit or drift.
Together, these dynamics confirm that the CAGM prior provides a sufficiently diverse and well-structured training signal for the model to learn meaningful clustering representations, laying the foundation for the strong zero-shot performance observed on real-world graphs in Section~\ref{sec:performance}.
We further evaluate different pre-training checkpoints directly on the benchmark graphs; Figure~\ref{fig:pt_realistic_nmi} shows the NMI trajectories on homophilous, heterophilous, and all ten graphs, with detailed analysis provided in Appendix~\ref{app:pretrain_dynamics}.

\begin{figure}[ht]
    \centering
    \includegraphics[width=0.49\linewidth]{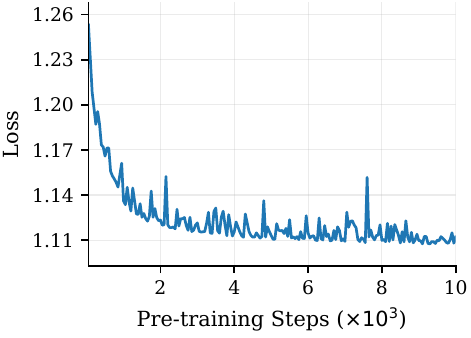}
    \hfill
    \includegraphics[width=0.49\linewidth]{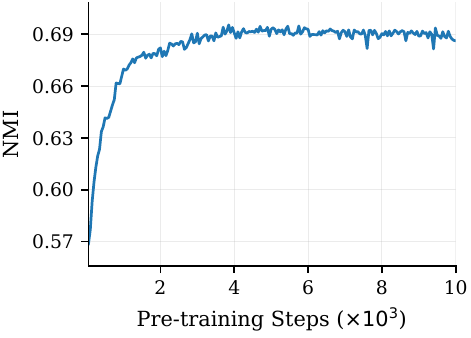}
    \caption{Validation loss and NMI on the fixed synthetic validation set evaluated every 50 steps during pre-training.}
    \label{fig:pretrain}
\end{figure}

\subsection{Representation Visualization}
\label{sec:vis}

To qualitatively assess what \texttt{OFAG} learns, we visualize node representations on the Photo dataset.
We run a single frozen forward pass to obtain the hyperspherical node embeddings, project them to $\mathbb{S}^2$ via PCA followed by re-normalization, and render the resulting 3D unit-sphere plot; Figure~\ref{fig:sphere} shows the raw node features and the \texttt{OFAG} embeddings side by side.
The raw features (Figure~\ref{fig:sphere}, left), when projected onto the sphere, form a diffuse, largely undifferentiated cloud with no visible cluster organization.
In contrast, the \texttt{OFAG} embeddings (Figure~\ref{fig:sphere}, right) exhibit well-separated, compact clusters distributed across the sphere surface.
This geometric reorganization is entirely zero-shot: the model has never observed Photo during pre-training, yet its encoder maps nodes from different categories into distinct hyperspherical regions.
The result provides direct visual evidence that pre-training on the CAGM prior instills a general inductive bias toward cluster-discriminative geometry, enabling the model to transfer structured representations to unseen real-world graphs without any task-specific adaptation.

\begin{figure}[ht]
    \centering
    \includegraphics[width=0.49\linewidth]{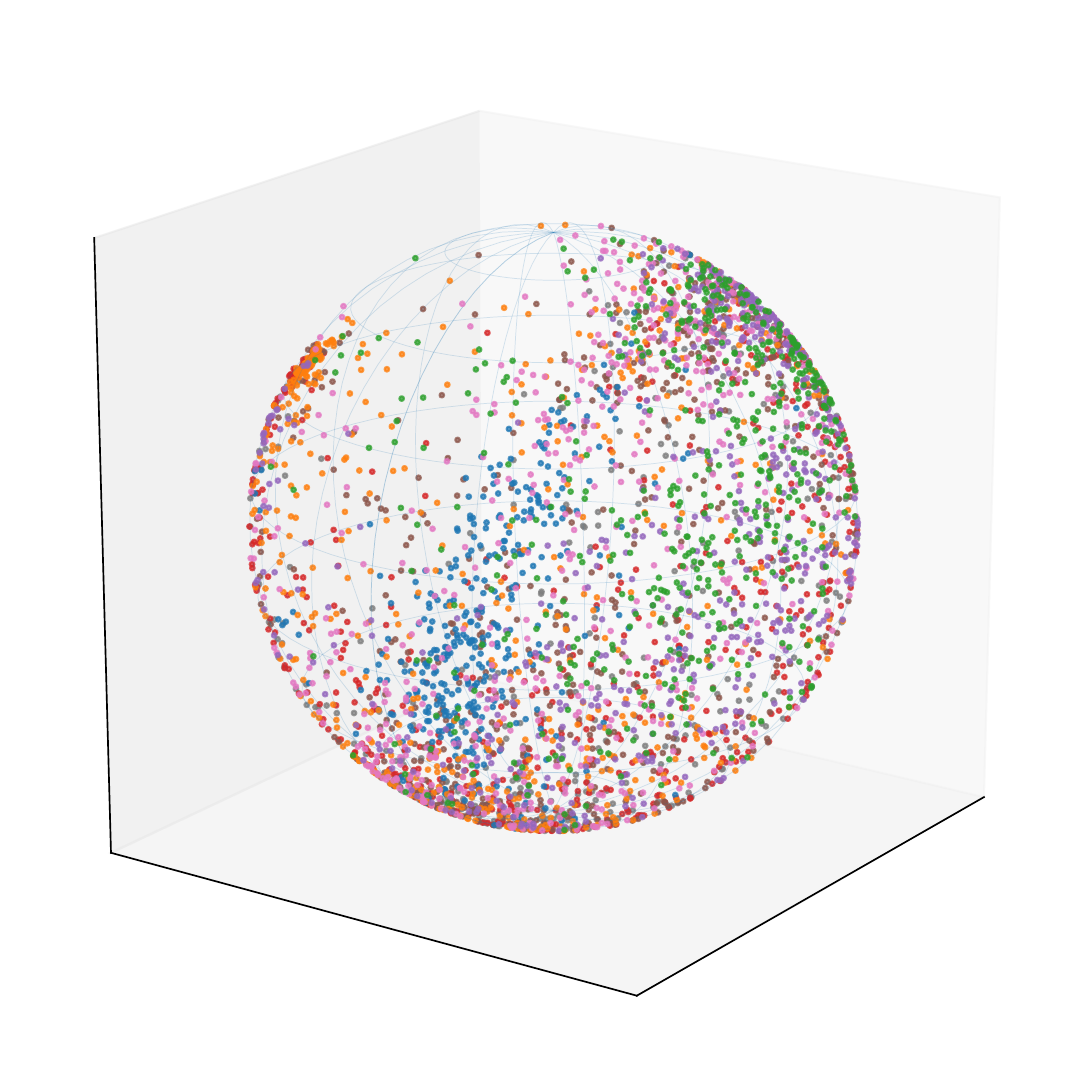}
    \hfill
    \includegraphics[width=0.49\linewidth]{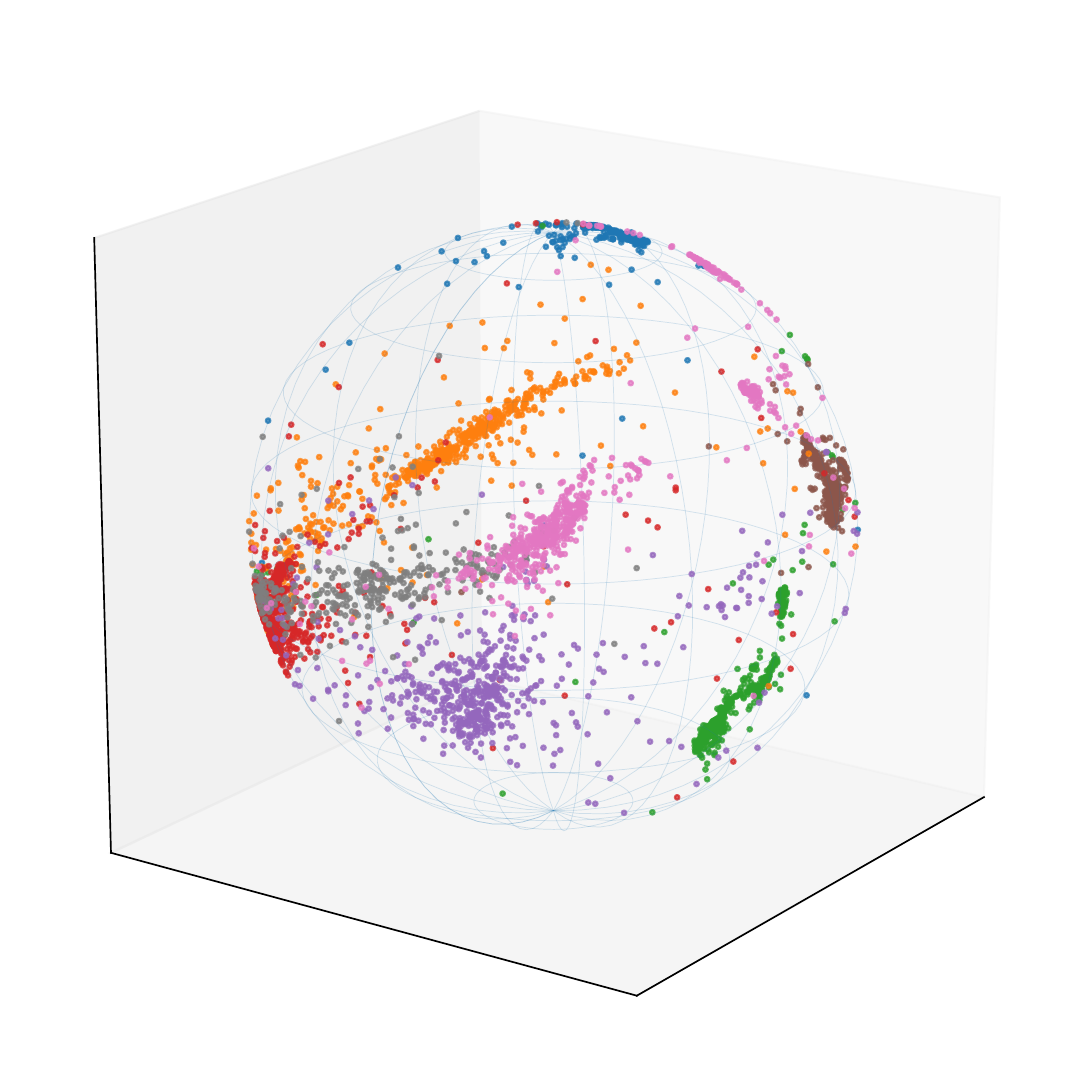}
    \caption{
    3D unit-sphere visualization of node representations on the Photo dataset, obtained by PCA projection and re-normalization. \textbf{Left:} raw node features are scattered uniformly with no discernible cluster structure. \textbf{Right:} \texttt{OFAG} embeddings form compact, well-separated clusters.
    }
    \label{fig:sphere}
\end{figure}

\vpara{Additional Results.}
Please kindly see the appendix for more analyses:
Appendix~\ref{app:pretrain_dynamics} analyzes how clustering quality on real graphs evolves across pre-training checkpoints; Appendices~\ref{app:ablation_L} and~\ref{app:ablation_lambda} examine sensitivity to the filter order and dispersion weight; Appendix~\ref{app:ablation_homophily} ablates the homophily coverage of the structural prior; and Appendix~\ref{app:complexity} provides a formal time complexity analysis.

\section{Conclusion}
\label{sec:conclusion}

We presented \texttt{OFAG}, the first one-for-all foundation model for attributed graph clustering.
By adapting the Prior-data Fitted Network paradigm to the unsupervised graph setting, \texttt{OFAG} is pre-trained once on synthetic graphs sampled from the CAGM prior and applied to arbitrary real-world graphs in a single forward pass.
Experiments on ten datasets show that a single frozen \texttt{OFAG} model achieves the best mean performance while being the fastest method overall.


\bibliographystyle{ACM-Reference-Format}
\bibliography{main}

\section*{Ethics and Privacy Statement}

This work proposes \texttt{OFAG}, a foundation model for unsupervised attributed graph clustering.
All benchmark datasets used in our experiments are publicly available and have been widely adopted in prior graph learning research; no new personal data was collected, and no human subjects were involved.
The pre-training data consists entirely of synthetic graphs sampled from the CAGM generative prior, which does not contain or encode any personal information.
The primary societal benefit of this work is efficiency: by eliminating per-graph model training and hyperparameter tuning, \texttt{OFAG} reduces the computational cost of graph clustering by orders of magnitude, lowering the barrier to deploying unsupervised graph analysis in resource-constrained settings.
As with any clustering method, downstream applications that group individuals (e.g., user segmentation or community detection in social networks) carry inherent fairness risks, since unsupervised partitions can reflect and amplify demographic biases present in the underlying graph structure and node attributes.
We encourage practitioners to audit clustering outputs for disparate impact before using them in high-stakes decisions.

\appendix

\section*{Appendix Overview}

\begin{description}[leftmargin=2em, labelindent=0em, itemsep=2pt]
  \item[Appendix~\ref{app:prior_details}] Details of the Synthetic Attributed Graph Prior
  \item[Appendix~\ref{app:thm_synthetic_real}] Proof of Theorem~1 (Synthetic-to-Real Generalization)
  \item[Appendix~\ref{app:prop_swifrt}] Proof of Proposition~1 (Dimension-Agnosticism and Feature-Permutation Equivariance)
  \item[Appendix~\ref{app:prop_vmf}] Proof of Proposition~2 (Compactness Loss as vMF MLE)
  \item[Appendix~\ref{app:prop_pfn}] Proof of Proposition~3 (\texttt{OFAG} as a Prior-Data Fitted Network)
  \item[Appendix~\ref{app:datasets}] Dataset Details
  \item[Appendix~\ref{app:baselines}] Baseline Details
  \item[Appendix~\ref{app:implementation}] Implementation and Evaluation Details
  \item[Appendix~\ref{app:pretrain_dynamics}] Pre-training Dynamics on Benchmark Graphs
  \item[Appendix~\ref{app:ablation_L}] Additional Experiments: Sensitivity to Filter Order $L$
  \item[Appendix~\ref{app:ablation_lambda}] Additional Experiments: Sensitivity to Dispersion Weight $\lambda$
  \item[Appendix~\ref{app:ablation_homophily}] Additional Experiments: Ablation on Homophily Coverage
  \item[Appendix~\ref{app:complexity}] Time Complexity Analysis
\end{description}

\begin{figure*}[t]
\centering
\subfigure[$N$=512, $E$=3,022, $K$=4, $h$=0.194]{\includegraphics[width=0.24\linewidth]{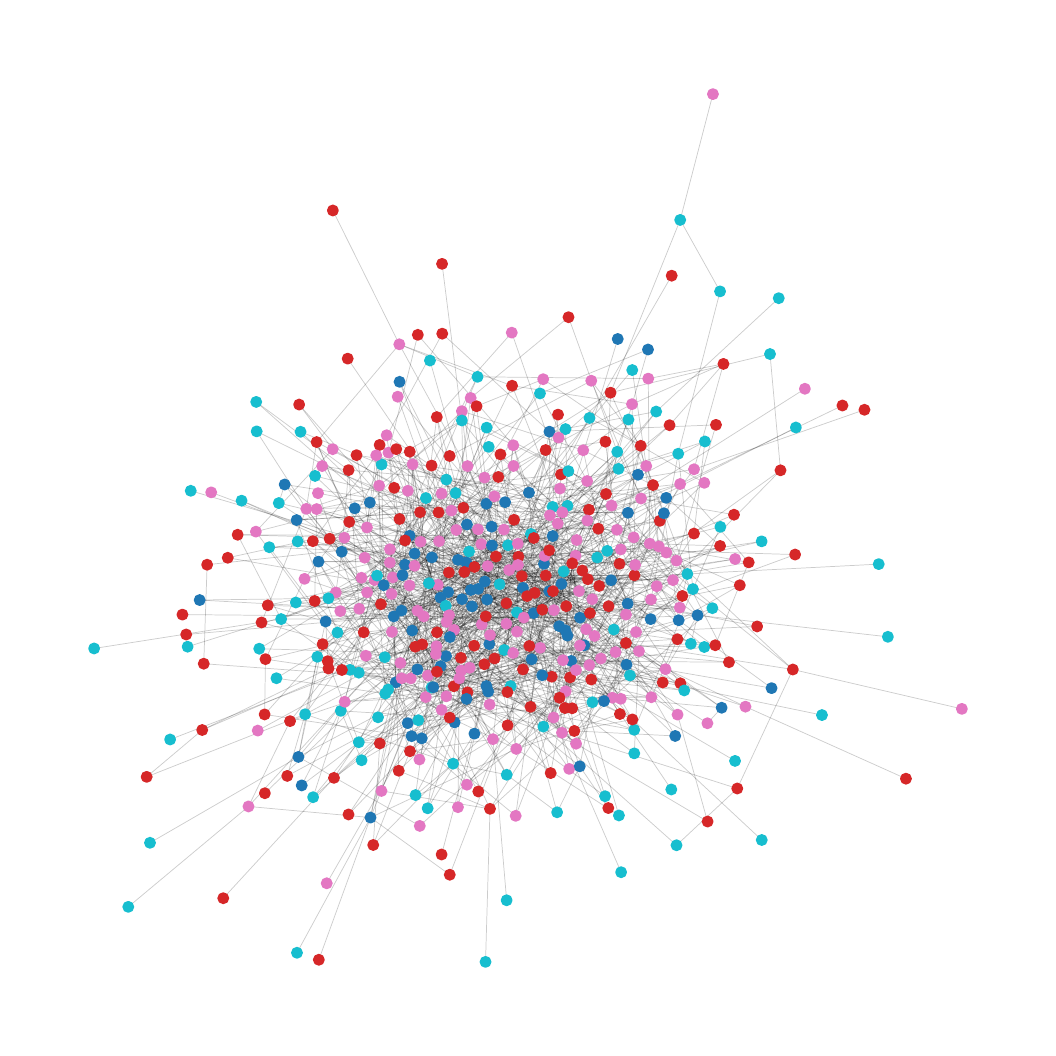}}
\subfigure[$N$=640, $E$=3,856, $K$=5, $h$=0.491]{\includegraphics[width=0.24\linewidth]{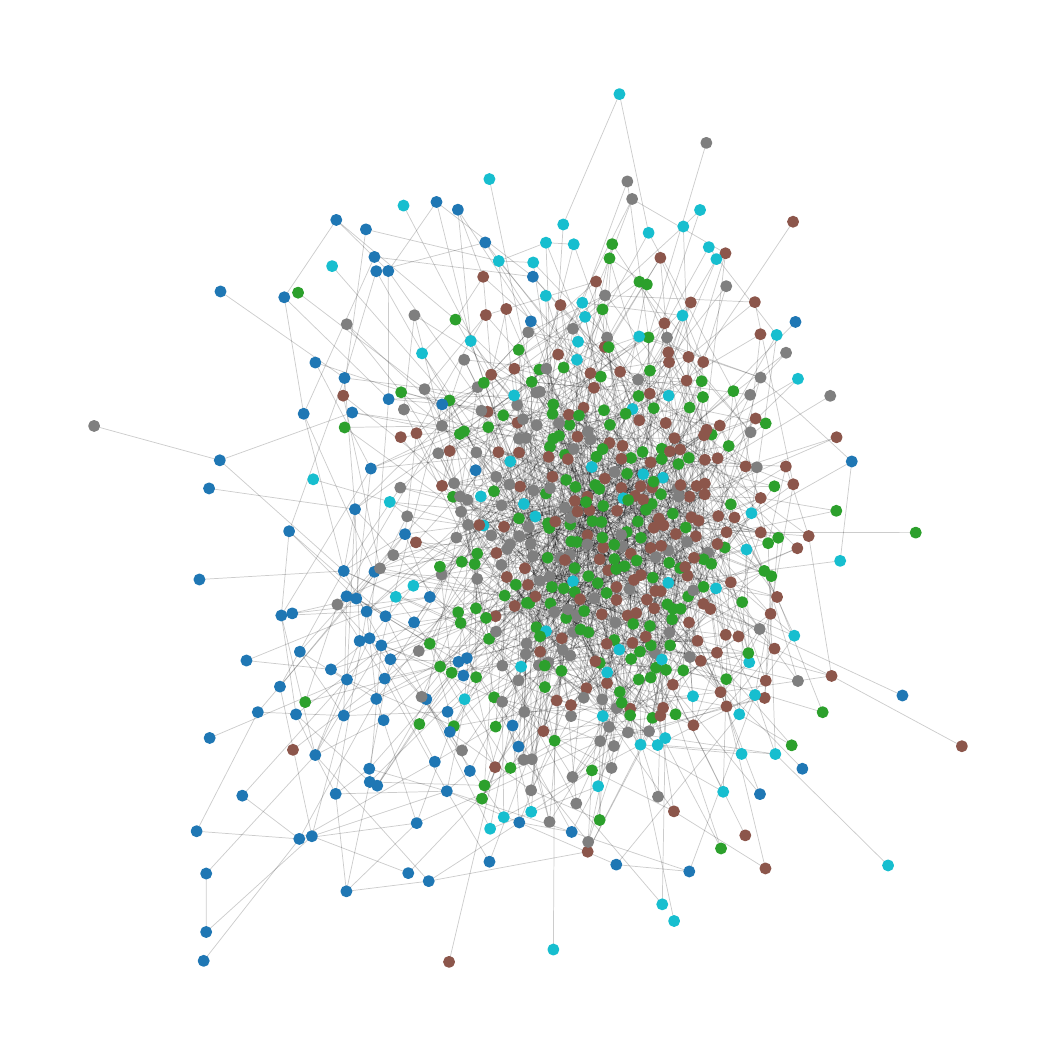}}
\subfigure[$N$=768, $E$=4,688, $K$=8, $h$=0.589]{\includegraphics[width=0.24\linewidth]{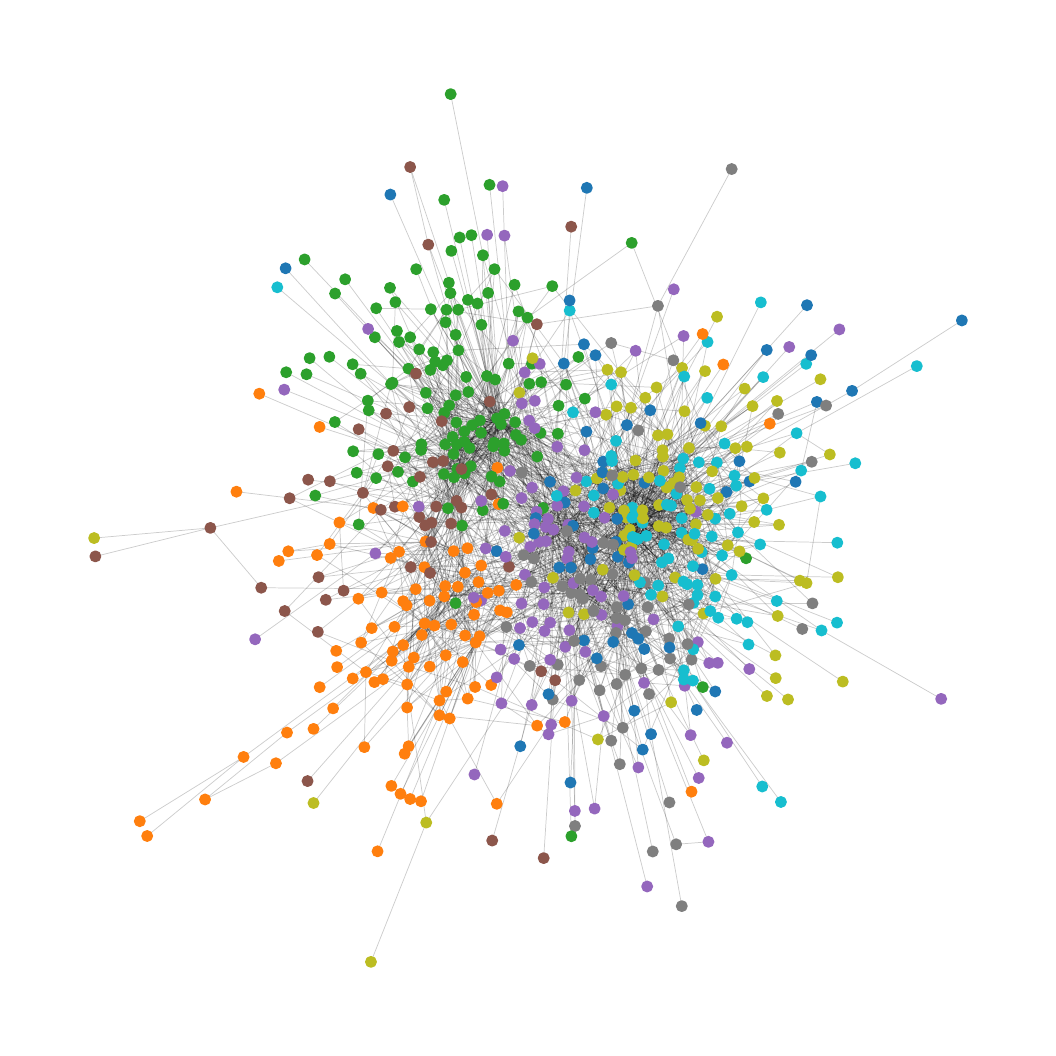}}
\subfigure[$N$=1,024, $E$=6,360, $K$=6, $h$=0.863]{\includegraphics[width=0.24\linewidth]{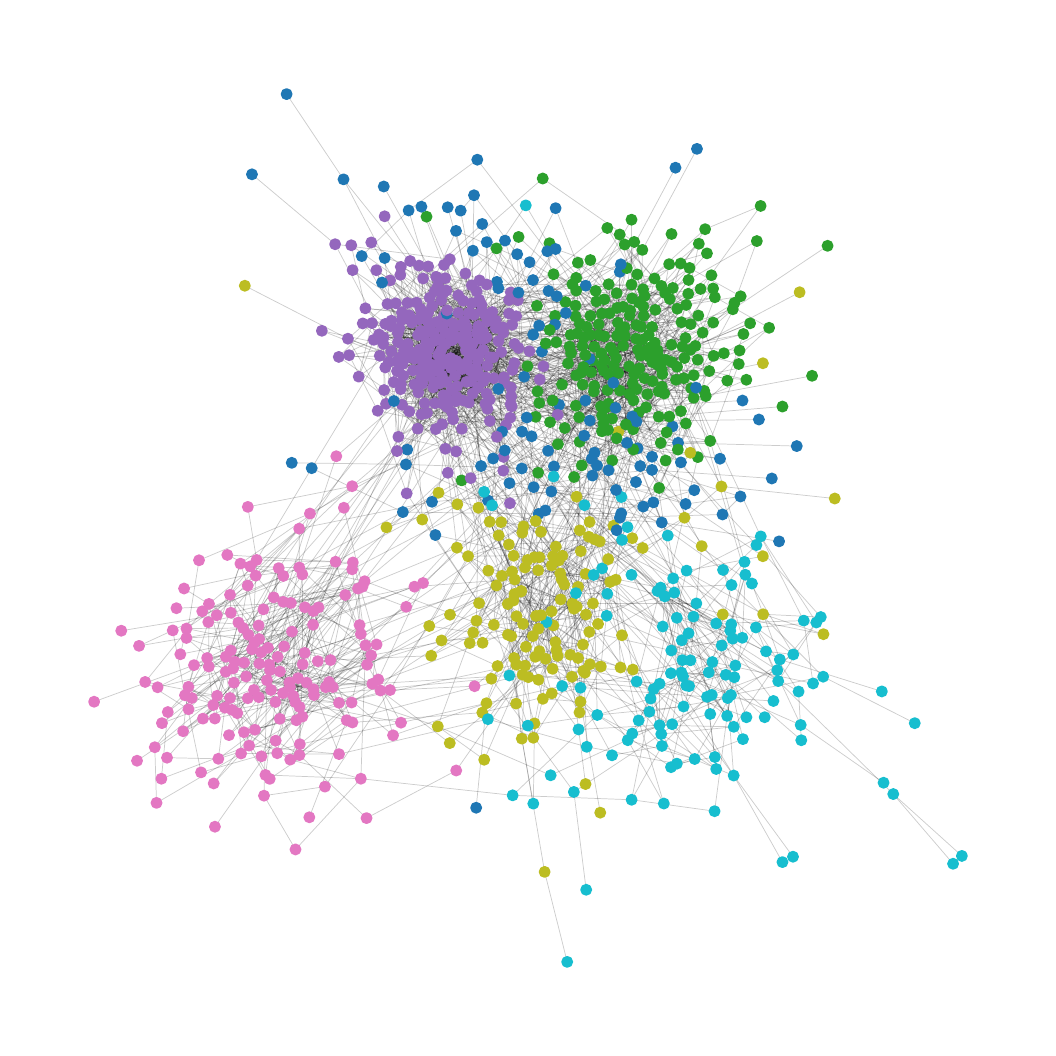}}
\caption{Representative synthetic graphs sampled from the CAGM prior, with varying numbers of nodes ($N$), edges ($E$), clusters ($K$), and edge homophily ($h$). Node colors indicate ground-truth cluster assignments. The four examples span the full structural spectrum: from strongly heterophilous ($h=0.194$, left) through near-neutral ($h=0.491$) and mildly homophilous ($h=0.589$) to strongly homophilous ($h=0.863$, right), illustrating the broad coverage of homophily regimes.}
\label{fig:cagm_examples}
\end{figure*}

\section{Details of Synthetic Attributed Graph Prior}
\label{app:prior_details}

This appendix describes the synthetic prior used for pre-training \texttt{OFAG}.
The purpose of the prior is not to exactly reproduce any single real-world graph family, but to generate a broad collection of attributed graph clustering tasks with diverse attribute distributions, graph structures, cluster imbalances, and modality reliabilities.
Figure~\ref{fig:cagm_examples} visualizes four representative synthetic graphs sampled from the CAGM prior, illustrating the diversity in size, cluster count, and structural regime.

\subsection{Latent Cluster Generation}

For each synthetic task, we sample the graph size, feature dimension, and number of clusters:
\begin{equation}
    N \sim p(N),
    \qquad
    F \sim p(F),
    \qquad
    K \sim p(K).
\end{equation}
We then sample cluster proportions from a Dirichlet distribution:
\begin{equation}
    \pi \sim \operatorname{Dirichlet}(\eta \mathbf{1}_K),
    \qquad
    y_i \sim \operatorname{Categorical}(\pi),
    \label{eq:app_cluster_prior}
\end{equation}
where $\eta$ controls cluster imbalance.
Sampling $\eta$ from a broad range allows the prior to generate both balanced and highly imbalanced clustering tasks.

\subsection{Attribute Generation}

Given cluster assignments, node attributes are generated from cluster-dependent latent variables.
For each cluster $k$, we sample a latent centroid and scale:
\begin{equation}
    \mu_k \sim \mathcal{N}(0,I),
    \qquad
    \Sigma_k \sim p(\Sigma).
\end{equation}
For a node $i$ assigned to cluster $k$, we sample
\begin{equation}
    h_i^X \mid y_i=k
    \sim
    t_{\nu}(\mu_k,\Sigma_k),
    \label{eq:app_latent_attribute}
\end{equation}
where $t_\nu$ denotes a multivariate Student-$t$ distribution.
The heavy-tailed distribution introduces outliers and non-Gaussian feature patterns.

Observed attributes are generated through a task-specific random decoder:
\begin{equation}
    u_i = W h_i^X + b + \epsilon_i,
    \qquad
    x_i = T(u_i),
    \label{eq:app_feature_decoder}
\end{equation}
where $W$, $b$, and $\epsilon_i$ are sampled for each task, and $T$ is a randomly selected observation transform.
Depending on the transform, the observed attributes may be continuous, binary, count-valued, categorical, sparse, or mixed-type.
We additionally inject irrelevant dimensions, feature noise, random scaling, and missing or sparse entries to increase heterogeneity.

The strength of the attribute signal is controlled by an attribute separability variable.
By varying the between-cluster distance relative to within-cluster variation, the prior produces tasks where attributes are strongly informative, weakly informative, or nearly uninformative.

\subsection{Graph Generation}

Edges are generated conditionally on cluster assignments, degree propensities, latent structural positions, and optional attribute similarity.
For $i<j$, we sample
\begin{equation}
    A_{ij} \sim \operatorname{Bernoulli}(p_{ij}),
\end{equation}
where
\begin{align}
    \operatorname{logit}(p_{ij})
    =
    \rho
    + \alpha_{\theta}(\theta_i+\theta_j)
    + \lambda_B B_{y_i,y_j}
    + \lambda_H \kappa(h_i^A,h_j^A)
    + \lambda_X \operatorname{sim}(x_i,x_j).
    \label{eq:app_edge_emission}
\end{align}
Here, $\rho$ controls global sparsity, $\theta_i$ is a node-specific degree propensity, $B$ is a cluster-pair affinity matrix, $h_i^A$ is a latent structural vector, $\kappa(\cdot,\cdot)$ is a structural similarity function, and $\operatorname{sim}(\cdot,\cdot)$ is an attribute similarity function.
The coefficients $\lambda_B$, $\lambda_H$, and $\lambda_X$ control how strongly edges depend on cluster identity, latent structural similarity, and observed attribute similarity.

To model degree heterogeneity, we sample
\begin{equation}
    \theta_i
    \sim
    \mathcal{N}
    \left(
        -\frac{1}{2}\sigma_\theta^2,
        \sigma_\theta^2
    \right),
    \qquad
    \sigma_\theta \sim p(\sigma_\theta).
    \label{eq:app_degree_propensity}
\end{equation}
This allows the prior to generate graphs with approximately regular degrees as well as graphs with hubs and heavy-tailed degree distributions.

\subsection{Structural Regimes}
\label{app:structural_regimes}

The block affinity matrix $B\in\mathbb{R}^{K\times K}$ is constructed from a continuously sampled edge homophily level, providing smooth coverage of both homophilous and heterophilous connectivity without committing to discrete hand-crafted regimes.

\vpara{Homophily sampling.}
For each synthetic task we first draw
\begin{equation}
    h \sim \mathrm{Uniform}(0.1,\,0.9),
\end{equation}
where $h$ is the target edge homophily, defined as the expected fraction of edges connecting nodes of the same cluster.

\vpara{Affinity matrix construction.}
Given $h$, the block affinity matrix is constructed as
\begin{equation}
    B = R + \beta(h)\,I_K,
    \qquad
    \beta(h) = s\cdot\operatorname{logit}(h) = s\cdot\log\!\frac{h}{1-h},
    \label{eq:affinity_construction}
\end{equation}
where $R\in\mathbb{R}^{K\times K}$ is a symmetric random matrix with entries $R_{kl}=R_{lk}\overset{\mathrm{iid}}{\sim}\mathcal{N}(0,\sigma_R^2)$, providing heterogeneous off-diagonal affinities across cluster pairs, and $s>0$ is a global strength hyperparameter.
The diagonal shift $\beta(h)I_K$ continuously biases intra-cluster connectivity relative to inter-cluster connectivity:
\begin{itemize}[leftmargin=*,itemsep=2pt,topsep=2pt]
    \item $h>0.5$ ($\beta>0$): diagonal entries are boosted, yielding \emph{homophilous} graphs in which same-cluster edges are more probable.
    \item $h<0.5$ ($\beta<0$): diagonal entries are suppressed, yielding \emph{heterophilous} graphs in which cross-cluster edges are more probable.
    \item $h=0.5$ ($\beta=0$): the matrix reduces to the pure random baseline $R$ with no systematic intra/inter-cluster bias.
\end{itemize}

\vpara{Coverage and rationale.}
Because $\operatorname{logit}(h)$ is monotone and unbounded, small changes in $h$ near $0.5$ produce mild biases while values close to the boundary produce strong homophilous or heterophilous signals, naturally generating a broad and calibrated distribution over structural regimes.
This formulation prevents the model from assuming that connected nodes should always belong to the same cluster, and exposes \texttt{OFAG} during pre-training to the full range of real-world homophily levels observed in attributed graph benchmarks.

\subsection{Modality Reliability}

Finally, we explicitly vary the reliability of attributes and topology.
Some tasks are attribute-dominant, some are structure-dominant, some require combining both modalities, and some contain misleading or weakly informative graph structure.
This is implemented by varying the attribute separability, graph signal strength, edge noise, degree heterogeneity, and the coupling coefficients in \eqref{eq:app_edge_emission}.
Through this construction, \texttt{OFAG} is exposed during pre-training to a broad range of attributed graph clustering scenarios.

\begin{table*}[t]
\centering
\caption{Clustering performance in ARI and F1 on homophilous and heterophilous attributed graphs. The \bestcap{}, \secondcap{}, and \thirdcap{} results in each column are highlighted. For Avg. Rank, lower is better; for all other columns, higher is better.}
\label{tab:ari_f1}
\scriptsize
\setlength{\tabcolsep}{2.2pt}
\resizebox{\textwidth}{!}{%
\begin{tabular}{llcccccccccccc}
\toprule
\multirow{2}{*}{Model} & \multirow{2}{*}{Metric} & \multicolumn{5}{c}{Homophilous Graphs} & \multicolumn{5}{c}{Heterophilous Graphs} & \multirow{2}{*}{Mean} & \multirow{2}{*}{Avg. Rank} \\
\cmidrule(lr){3-7} \cmidrule(lr){8-12}
& & Cora & Photo & ArXiv & Reddit & Products & HM & Flickr & MAG & Pokec & WebTopic & & \\
\midrule
\multirow{2}{*}{KMeans} & ARI & \res{6.95}{2.61} & \res{20.01}{1.01} & \res{7.23}{0.15} & \res{2.93}{0.05} & \res{9.22}{0.33} & \res{2.96}{0.09} & \res{1.06}{0.15} & \res{2.77}{0.03} & \res{0.23}{0.00} & \res{0.87}{0.04} & 5.42 & 16.30 \\
    & F1 & \res{25.99}{2.91} & \res{42.68}{3.32} & \res{12.97}{0.18} & \res{7.25}{0.08} & \res{12.77}{0.41} & \res{12.50}{0.20} & \res{15.16}{0.85} & \res{5.71}{0.07} & \res{0.98}{0.01} & \res{4.64}{0.04} & 14.06 & 16.45 \\
\addlinespace[1pt]
\multirow{2}{*}{Node2Vec} & ARI & \res{33.61}{2.49} & \res{57.11}{1.39} & \res{16.69}{0.73} & \res{64.34}{1.64} & \res{18.53}{1.05} & \res{1.84}{0.07} & \res{1.52}{0.00} & \res{4.75}{0.07} & \second{\res{9.27}{0.08}} & \res{1.53}{0.03} & 20.92 & 10.90 \\
    & F1 & \res{61.66}{2.52} & \res{69.33}{1.18} & \res{20.96}{0.29} & \res{68.61}{2.12} & \third{\res{25.41}{0.24}} & \res{11.65}{0.46} & \res{19.68}{0.01} & \res{8.90}{0.09} & \second{\res{12.66}{0.15}} & \res{7.40}{0.16} & 30.63 & 8.70 \\
\addlinespace[1pt]
\midrule
\multirow{2}{*}{SSGC} & ARI & \res{41.25}{1.69} & \res{59.17}{3.81} & \best{\res{33.04}{0.29}} & \res{33.61}{1.29} & \res{18.84}{0.40} & \res{3.18}{0.02} & \res{3.29}{0.11} & \res{5.84}{0.09} & \res{0.42}{0.01} & \res{1.75}{0.04} & 20.04 & 8.45 \\
    & F1 & \res{65.16}{1.32} & \res{70.17}{3.78} & \second{\res{25.44}{0.59}} & \res{29.15}{1.33} & \res{24.01}{0.71} & \res{14.84}{0.13} & \res{18.64}{0.48} & \third{\res{9.78}{0.05}} & \res{1.60}{0.03} & \res{5.66}{0.04} & 26.45 & 8.50 \\
\addlinespace[1pt]
\multirow{2}{*}{SAGSC} & ARI & \res{38.83}{0.02} & \res{43.66}{0.01} & \res{20.69}{0.69} & \third{\res{65.93}{1.44}} & \res{19.36}{0.29} & \third{\res{4.12}{0.08}} & \res{3.27}{0.00} & \res{4.89}{0.11} & \best{\res{11.58}{0.18}} & \res{2.91}{0.10} & 21.52 & 8.00 \\
    & F1 & \res{61.03}{0.02} & \res{69.40}{0.01} & \third{\res{25.37}{0.85}} & \third{\res{70.43}{1.60}} & \second{\res{25.56}{0.53}} & \second{\res{16.08}{0.10}} & \res{20.97}{0.01} & \res{9.67}{0.20} & \best{\res{16.12}{0.26}} & \third{\res{9.00}{0.30}} & \second{32.36} & \third{4.90} \\
\addlinespace[1pt]
\multirow{2}{*}{MS2CAG} & ARI & \res{45.76}{0.98} & \best{\res{63.88}{0.54}} & \res{28.84}{0.18} & \res{56.56}{2.89} & \res{17.97}{0.65} & \res{3.50}{0.11} & \res{3.84}{0.03} & \res{5.16}{0.09} & \res{0.42}{0.03} & \res{2.36}{0.03} & 22.83 & 7.15 \\
    & F1 & \res{69.17}{0.65} & \second{\res{74.94}{0.95}} & \res{24.58}{0.50} & \res{55.37}{1.78} & \res{24.67}{0.31} & \res{12.95}{0.27} & \res{21.81}{0.03} & \second{\res{9.97}{0.14}} & \res{1.51}{0.02} & \res{8.52}{0.07} & 30.35 & 6.40 \\
\addlinespace[1pt]
\midrule
\multirow{2}{*}{DAEGC} & ARI & \res{38.36}{2.67} & \res{56.37}{0.02} & \res{17.53}{0.40} & \res{13.43}{1.24} & \res{2.35}{0.90} & \res{3.47}{0.07} & \res{4.61}{0.10} & \res{4.29}{0.36} & \res{0.61}{0.04} & \res{2.58}{0.38} & 14.36 & 10.85 \\
    & F1 & \res{64.54}{2.32} & \best{\res{76.26}{0.01}} & \res{20.64}{0.28} & \res{8.23}{0.93} & \res{8.02}{1.81} & \res{14.73}{0.05} & \res{19.24}{0.09} & \res{7.25}{0.43} & \res{1.73}{0.07} & \res{7.34}{0.19} & 22.80 & 10.90 \\
\addlinespace[1pt]
\multirow{2}{*}{DinkNet} & ARI & \res{52.45}{0.41} & \res{53.54}{0.23} & \res{14.93}{0.50} & \res{33.70}{1.24} & \res{13.67}{0.40} & \res{2.96}{0.03} & \res{1.56}{0.33} & \res{5.10}{0.04} & \res{0.29}{0.02} & \res{1.86}{0.25} & 18.01 & 11.85 \\
    & F1 & \res{70.12}{0.59} & \res{67.20}{0.27} & \res{16.94}{0.73} & \res{27.89}{0.94} & \res{14.14}{0.12} & \res{14.30}{0.26} & \res{16.50}{0.40} & \res{7.65}{0.11} & \res{1.65}{0.01} & \res{5.71}{0.10} & 24.21 & 12.40 \\
\addlinespace[1pt]
\multirow{2}{*}{MinCut} & ARI & \res{28.60}{2.86} & \res{50.99}{2.33} & \res{17.47}{0.98} & \res{25.42}{1.67} & \res{10.64}{0.32} & \res{1.22}{0.30} & \res{3.73}{0.66} & \res{5.56}{0.15} & \res{0.61}{0.04} & \res{2.00}{0.20} & 14.62 & 12.65 \\
    & F1 & \res{51.88}{6.97} & \res{66.99}{2.97} & \res{21.85}{0.73} & \res{26.39}{5.74} & \res{17.53}{0.54} & \res{9.54}{0.68} & \res{21.22}{1.17} & \res{8.04}{0.12} & \res{2.30}{0.05} & \res{6.92}{0.46} & 23.27 & 12.00 \\
\addlinespace[1pt]
\multirow{2}{*}{DMoN} & ARI & \res{33.35}{3.35} & \res{51.55}{2.41} & \res{15.44}{0.37} & \res{32.39}{1.40} & \res{10.93}{0.41} & \res{2.05}{0.28} & \res{5.01}{0.52} & \res{4.93}{0.02} & \res{1.17}{0.04} & \res{2.43}{0.21} & 15.93 & 11.30 \\
    & F1 & \res{49.57}{3.84} & \res{70.86}{2.44} & \res{20.37}{0.61} & \res{26.52}{0.76} & \res{16.43}{1.13} & \res{10.58}{0.44} & \third{\res{22.76}{0.48}} & \res{8.97}{0.14} & \res{3.04}{0.04} & \res{7.80}{0.21} & 23.69 & 9.95 \\
\addlinespace[1pt]
\multirow{2}{*}{Neuromap} & ARI & \res{36.69}{3.25} & \res{53.01}{3.67} & \res{21.92}{1.08} & \res{25.32}{0.26} & \res{16.52}{1.56} & \res{1.45}{0.07} & \res{4.99}{1.28} & \best{\res{9.85}{0.36}} & \third{\res{1.84}{0.11}} & \third{\res{3.31}{0.56}} & 17.49 & 8.50 \\
    & F1 & \res{54.69}{3.95} & \res{69.13}{3.76} & \res{21.29}{0.84} & \res{25.70}{0.43} & \res{13.09}{1.04} & \res{9.37}{0.18} & \res{22.39}{1.57} & \res{6.20}{0.19} & \res{2.62}{0.14} & \res{7.96}{0.39} & 23.24 & 11.80 \\
\addlinespace[1pt]
\midrule
\multirow{2}{*}{GAE} & ARI & \res{45.63}{0.15} & \res{52.39}{0.12} & \res{15.95}{0.40} & \res{27.88}{0.46} & \res{14.49}{0.25} & \second{\res{4.21}{0.23}} & \second{\res{5.53}{0.30}} & \res{5.34}{0.08} & \res{0.34}{0.01} & \res{1.98}{0.06} & 17.37 & 9.90 \\
    & F1 & \res{64.48}{0.11} & \res{67.52}{0.09} & \res{18.57}{0.22} & \res{26.89}{0.50} & \res{18.35}{0.47} & \third{\res{15.20}{0.22}} & \res{18.83}{0.13} & \res{8.97}{0.09} & \res{1.53}{0.02} & \res{5.81}{0.14} & 24.61 & 11.25 \\
\addlinespace[1pt]
\multirow{2}{*}{DGI} & ARI & \res{52.03}{3.22} & \res{58.38}{0.46} & \res{22.77}{0.28} & \res{59.28}{0.54} & \res{14.67}{0.21} & \res{3.11}{0.06} & \res{4.98}{0.01} & \res{5.55}{0.07} & \res{0.32}{0.00} & \res{2.18}{0.06} & 22.33 & 8.50 \\
    & F1 & \res{70.07}{2.56} & \res{72.77}{1.71} & \res{21.47}{0.30} & \res{51.97}{0.79} & \res{17.87}{0.32} & \res{14.20}{0.23} & \res{19.92}{0.15} & \res{8.56}{0.05} & \res{1.59}{0.01} & \res{7.10}{0.06} & 28.55 & 9.25 \\
\addlinespace[1pt]
\multirow{2}{*}{CCASSG} & ARI & \res{51.84}{3.32} & \res{51.80}{4.28} & \res{19.18}{0.18} & \res{31.28}{0.32} & \third{\res{22.43}{0.70}} & \res{3.20}{0.05} & \res{3.34}{0.03} & \res{6.11}{0.05} & \res{0.23}{0.00} & \res{1.40}{0.29} & 19.08 & 10.15 \\
    & F1 & \third{\res{72.29}{1.48}} & \res{69.69}{1.95} & \res{23.78}{0.25} & \res{30.09}{0.32} & \res{22.83}{0.39} & \res{14.29}{0.24} & \res{16.67}{0.26} & \res{8.56}{0.06} & \res{0.98}{0.01} & \res{5.17}{0.16} & 26.44 & 10.50 \\
\addlinespace[1pt]
\multirow{2}{*}{S3GC} & ARI & \res{48.12}{2.56} & \res{59.09}{2.86} & \res{27.08}{0.55} & \second{\res{83.27}{2.36}} & \second{\res{23.57}{0.28}} & \res{3.30}{0.05} & \third{\res{5.45}{0.00}} & \second{\res{7.51}{0.11}} & \res{0.67}{0.01} & \res{2.56}{0.04} & \second{26.06} & \second{4.70} \\
    & F1 & \res{69.04}{1.57} & \res{70.21}{2.55} & \res{24.09}{0.70} & \second{\res{71.35}{1.60}} & \res{25.19}{0.41} & \res{14.59}{0.29} & \second{\res{23.21}{0.00}} & \res{7.91}{0.05} & \res{2.02}{0.01} & \res{8.07}{0.02} & \third{31.57} & 6.50 \\
\addlinespace[1pt]
\multirow{2}{*}{MAGI} & ARI & \second{\res{53.51}{1.09}} & \res{58.53}{0.32} & \third{\res{30.38}{1.33}} & \res{62.27}{1.24} & \res{16.67}{0.37} & \res{2.46}{0.27} & \res{3.16}{0.27} & \res{6.44}{0.06} & \res{1.06}{0.01} & \res{2.92}{0.29} & \third{23.74} & \third{6.60} \\
    & F1 & \second{\res{72.65}{0.36}} & \res{70.82}{0.15} & \res{25.17}{0.59} & \res{55.26}{1.62} & \res{21.12}{0.19} & \res{14.33}{0.78} & \res{20.81}{0.36} & \res{9.46}{0.06} & \res{2.90}{0.02} & \best{\res{9.53}{0.27}} & 30.20 & 5.40 \\
\addlinespace[1pt]
\multirow{2}{*}{NS4GC} & ARI & \best{\res{56.63}{1.20}} & \second{\res{62.64}{0.62}} & \second{\res{31.84}{2.30}} & \res{39.12}{0.19} & \best{\res{23.82}{0.23}} & \best{\res{4.72}{0.10}} & \res{3.37}{0.01} & \third{\res{6.85}{0.06}} & \res{0.74}{0.01} & \best{\res{3.86}{0.05}} & 23.36 & \best{3.60} \\
    & F1 & \best{\res{72.81}{0.38}} & \third{\res{73.73}{1.11}} & \best{\res{26.67}{0.57}} & \res{34.95}{0.63} & \best{\res{26.52}{0.28}} & \best{\res{17.23}{0.28}} & \res{19.44}{0.03} & \res{9.14}{0.08} & \res{2.33}{0.02} & \res{8.80}{0.15} & 29.16 & \second{4.30} \\
\addlinespace[1pt]
\midrule
\multirow{2}{*}{\texttt{OFAG}} & ARI & \third{\res{53.23}{0.00}} & \third{\res{61.92}{0.00}} & \res{25.22}{0.29} & \best{\res{88.40}{0.00}} & \res{19.87}{0.40} & \res{3.21}{0.05} & \best{\res{5.96}{0.00}} & \res{6.33}{0.11} & \res{1.43}{0.01} & \second{\res{3.84}{0.00}} & \best{26.94} & \best{3.60} \\
    & F1 & \res{71.72}{0.00} & \res{72.93}{0.00} & \res{24.38}{0.38} & \best{\res{81.17}{0.00}} & \res{25.02}{0.53} & \res{13.54}{0.14} & \best{\res{23.89}{0.00}} & \best{\res{10.28}{0.04}} & \third{\res{3.12}{0.02}} & \second{\res{9.34}{0.00}} & \best{33.54} & \best{3.80} \\
\bottomrule
\end{tabular}%
}
\end{table*}

\section{Proof of Theorem~\ref{thm:synthetic_real}}
\label{app:thm_synthetic_real}

\begin{proof}
We establish the approximation separately for condition (i) (node features) and condition (ii) (edges), then combine.

\textbf{Condition (i): node feature approximation.}
Let $p_{\mathrm{real}}(X\mid y)=\sum_{k=1}^{K}\pi_k p_k(X)$ with $y$-dependent components $p_k$.
By the universal approximation theorem, for each $k$ there exists a neural network $\Phi_k$ transforming a Gaussian source into an approximation of $p_k$ with arbitrarily small total variation error.
A unified $\Phi$ can implement all $\Phi_k$ conditioned on $k$, approximating the full mixture.
In the CAGM prior, cluster-conditioned latent variables are drawn from $t_\nu(\mu_k,\Sigma_k)$ and passed through a task-specific random decoder $T$ (Appendix~\ref{app:prior_details}).
Over a sufficiently rich family of decoders (e.g., randomized MLP transformations), $T$ implements $\Phi_k$ for each $k$, and the Dirichlet-sampled proportions $\pi$ cover arbitrary mixing weights.

\textbf{Condition (ii): edge distribution approximation.}
Any pairwise edge probability model $p(A_{ij}=1\mid y,X)=\sigma(s_{ij})$ is determined by scores $s_{ij}$ depending on cluster membership and/or node features.
The CAGM edge emission~\eqref{eq:app_edge_emission} parameterizes
\begin{equation}
    \operatorname{logit}(p_{ij})
    = \rho + \alpha_\theta(\theta_i+\theta_j) + \lambda_B B_{y_i,y_j} + \lambda_H\kappa(h_i^A,h_j^A) + \lambda_X\operatorname{sim}(x_i,x_j).
\end{equation}
The block affinity matrix $B\in\mathbb{R}^{K\times K}$ captures arbitrary cluster-conditioned connectivity, spanning homophilous to heterophilous regimes via the construction in Appendix~\ref{app:structural_regimes}.
The degree propensities $\theta_i$ model heterogeneous node degrees, and the feature-similarity term $\lambda_X\operatorname{sim}(x_i,x_j)$ encodes dependence of edges on observed attributes.
By setting these coefficients appropriately, any target $p(A\mid y,X)$ satisfying condition (ii) can be approximated.

\textbf{Combining.}
For any $\varepsilon>0$, set approximation errors in each part to at most $\varepsilon/2$.
By the triangle inequality for total variation, $\mathrm{TV}(\Phi_\# p_{\mathrm{syn}}(\mathcal{G}),\,p_{\mathrm{real}}(\mathcal{G}))\leq\varepsilon$.
\end{proof}

\section{Proof of Proposition~\ref{prop:swifrt}}
\label{app:prop_swifrt}

\begin{proof}
\textbf{(i) Parameter independence from $F$.}
The shared encoder $g_\theta$ consists of: a scalar value embedding $\phi_{\mathrm{val}}:\mathbb{R}\to\mathbb{R}^d$; learnable filter-index embeddings $\{p_\ell\}_{\ell=0}^L\subset\mathbb{R}^d$; a Transformer encoder operating over a sequence of length $L+2$; and a scalar decoder $\phi_{\mathrm{out}}:\mathbb{R}^d\to\mathbb{R}$.
None of these modules depend on $F$; the total parameter count is
$|\phi_{\mathrm{val}}|+(L+1)d+|\text{Transformer}|+|\phi_{\mathrm{out}}|,$
which is fixed regardless of the feature dimension $F$ of the input graph.
Although $g_\theta$ is applied $N\times F$ times in total during a forward pass, the parameter count does not change.

\textbf{(ii) Feature-permutation equivariance.}
Let $\sigma$ be a permutation of $\{1,\ldots,F\}$ with $(X_\sigma)_{:,f}=X_{:,\sigma(f)}$.
For each node $i$ and feature index $f$, the filter response profile under the permuted input satisfies
\begin{equation}
\begin{aligned}
    r_{i,f}^{(\sigma)}
    &= \bigl[(\Gamma_0 X_\sigma)_{i,f},\,\ldots,\,
             (\Gamma_L X_\sigma)_{i,f}\bigr] \\
    &= \bigl[(\Gamma_0 X)_{i,\sigma(f)},\,\ldots,\,
             (\Gamma_L X)_{i,\sigma(f)}\bigr]
     = r_{i,\sigma(f)},
\end{aligned}
\end{equation}
where the second equality holds because the graph filters $\Gamma_\ell(A)=\hat{A}^\ell$ act only on the node dimension and are independent of feature indices.
Since $g_\theta$ is applied independently to each $(i,f)$ pair without any cross-channel interaction, the re-encoded value for the permuted input is
$Z_{i,f}^{(\sigma)}=g_\theta(r_{i,f}^{(\sigma)})=g_\theta(r_{i,\sigma(f)})=Z_{i,\sigma(f)}$.
Therefore $\operatorname{SwiFRT}_\theta(X_\sigma,A)=\operatorname{SwiFRT}_\theta(X,A)_\sigma$.
\end{proof}

\section{Proof of Proposition~\ref{prop:vmf_mle}}
\label{app:prop_vmf}

\begin{proof}
Let $z_i\in S^{d-1}$ be the $\ell_2$-normalized embedding of node $i$, and let $\mu_k\in S^{d-1}$ be the prototype for cluster $k\in\{1,\ldots,K\}$.

\textbf{Von Mises--Fisher mixture model.}
Assume that embeddings within cluster $k$ are drawn from a von Mises--Fisher (vMF) distribution:
\begin{equation}
    p(z_i\mid y_i=k;\mu_k,\kappa) = Z_d(\kappa)\exp(\kappa\mu_k^\top z_i),
    \label{eq:vmf_density_app}
\end{equation}
where $Z_d(\kappa)=\kappa^{d/2-1}/\bigl((2\pi)^{d/2}I_{d/2-1}(\kappa)\bigr)$ is the normalization constant with $I_\nu$ denoting the modified Bessel function of the first kind of order $\nu$.
Note that $Z_d(\kappa)$ depends only on $\kappa$ and $d$, not on $\mu_k$.

\textbf{Posterior probability.}
Assuming equal mixing weights $\pi_k=1/K$, the posterior probability of cluster $k$ given $z_i$ is:
\begin{align}
    p(y_i=k\mid z_i;\{\mu_k\},\kappa)
    &= \frac{\pi_k\,Z_d(\kappa)\exp(\kappa\mu_k^\top z_i)}{\sum_{k'=1}^K \pi_{k'}\,Z_d(\kappa)\exp(\kappa\mu_{k'}^\top z_i)} \notag \\
    &= \frac{\exp(\mu_k^\top z_i/\tau)}{\sum_{k'=1}^K\exp(\mu_{k'}^\top z_i/\tau)},
    \label{eq:vmf_post_app}
\end{align}
where $\tau=1/\kappa$ and $Z_d(\kappa)$ cancels identically in numerator and denominator (since $\pi_k=1/K$ for all $k$).
This is precisely the form given in Eq.~\eqref{eq:vmf_posterior}.

\textbf{Connection to $\mathcal{L}_{\mathrm{comp}}$.}
The complete-data negative log-likelihood under the vMF mixture is:
\begin{equation}
\begin{aligned}
    &-\frac{1}{N}\sum_{i=1}^N\log p(y_i\mid z_i;\{\mu_k\},\kappa)\\
    &= -\frac{1}{N}\sum_{i=1}^N\log\frac{\exp(\mu_{y_i}^\top z_i/\tau)}{\sum_{k'=1}^K\exp(\mu_{k'}^\top z_i/\tau)}
    = \mathcal{L}_{\mathrm{comp}}.
\end{aligned}
\end{equation}
Therefore, minimizing $\mathcal{L}_{\mathrm{comp}}$ with respect to the encoder parameters $\theta$ (which determine $\{z_i\}$ through the forward pass) is equivalent to maximizing the complete-data log-likelihood under the $K$-component vMF mixture, i.e., to MLE on $\{\mu_k\}$ with concentration $\kappa=1/\tau$.

\begin{remark}
The cancellation of $Z_d(\kappa)$ in Eq.~\eqref{eq:vmf_post_app} is a distinctive property of the vMF distribution: its normalization constant does not depend on the mean direction $\mu_k$. This property does not hold for Gaussian distributions with class-specific covariances, where the log-determinant $\log|\Sigma_k|$ would remain in the objective and introduce additional terms.
\end{remark}
\end{proof}

\section{Proof of Proposition~\ref{prop:ofag_pfn}}
\label{app:prop_pfn}

\begin{proof}
The Prior-Data Negative Log-Likelihood (Prior-Data NLL) for a PFN is defined as~\citep{PFN, PFNPosition}:
\begin{equation}
    \mathcal{L}_{\mathrm{PFN}}(\theta)
    = \mathbb{E}_{\mathcal{D}_{\mathrm{c}}\cup\{(x,y)\}\sim p(\mathcal{D})}
    \bigl[-\log q_\theta(y\mid x,\mathcal{D}_{\mathrm{c}})\bigr].
    \label{eq:pfn_nll_app}
\end{equation}

\textbf{Identifying context and query.}
Since no cluster labels are available at inference time, the PFN context is $\mathcal{D}_{\mathrm{c}}=\emptyset$.
The query set is taken to be all nodes in a synthetic graph $(A,X,y)\sim p(\mathcal{D})$.
Each node $i$ corresponds to a query point $x=z_i$ with label $y=y_i$, where $z_i=Z_{i,:}/\|Z_{i,:}\|_2$ is the $\ell_2$-normalized output of SwiFRT.

\textbf{Identifying the predictive distribution.}
The model's predictive distribution is identified as the vMF posterior derived in Appendix~\ref{app:prop_vmf}:
\begin{equation}
\begin{aligned}
    q_\theta(y_i\mid z_i,\emptyset)
    &= p(y_i\mid z_i;\{\hat{\mu}_k\},\tau) \\
    &= \frac{\exp(\hat{\mu}_{y_i}^\top z_i/\tau)}
            {\sum_{k=1}^K \exp(\hat{\mu}_k^\top z_i/\tau)},
\end{aligned}
\end{equation}
where the prototypes $\{\hat{\mu}_k\}$ are computed on-the-fly from the current embeddings and synthetic labels via Eq.~\eqref{eq:prototype}.

\textbf{Connecting $\mathcal{L}_{\mathrm{PFN}}$ to $\mathcal{L}_{\mathrm{comp}}$.}
Substituting into Eq.~\eqref{eq:pfn_nll_app} with $\mathcal{D}_{\mathrm{c}}=\emptyset$ and averaging over all query nodes:
\begin{equation}
\begin{aligned}
    \mathcal{L}_{\mathrm{PFN}}(\theta)
    &= \mathbb{E}_{(A,X,y)\sim p(\mathcal{D})}
    \left[-\frac{1}{N}\sum_{i=1}^N\log q_\theta(y_i\mid z_i)\right]\\
    &= \mathbb{E}_{(A,X,y)\sim p(\mathcal{D})}\bigl[\mathcal{L}_{\mathrm{comp}}(\theta;A,X,y)\bigr].
\end{aligned}
\end{equation}

\textbf{Consequence for posterior approximation.}
By the equivalence theorem of~\citet{PFN}, minimizing the Prior-Data NLL is equivalent to minimizing the expected Kullback--Leibler divergence
\begin{equation}
    \mathbb{E}_{(A,X,y)\sim p(\mathcal{D})}
    \Bigl[D_{\mathrm{KL}}\!\bigl(p(\cdot\mid z,\emptyset)\,\big\|\,q_\theta(\cdot\mid z,\emptyset)\bigr)\Bigr]
\end{equation}
between $q_\theta$ and the true posterior predictive distribution $p(\cdot\mid z,\emptyset)$ induced by $p(\mathcal{D})$.
Therefore, a sufficiently expressive \texttt{OFAG} trained to convergence implicitly performs Bayesian model averaging over the space of vMF mixture clustering hypotheses consistent with $p(\mathcal{D})$, without ever computing a posterior explicitly.
\end{proof}

\begin{table*}[t]
    \centering
    \caption{Summary of the benchmark datasets.
    The collection spans multiple graph scales, domains, and feature modalities.
    $\mathcal{H}_e$ and $\mathcal{H}_n$ denote edge homophily and node homophily, respectively.
    }
    \label{tab:datasets}
    \resizebox{\textwidth}{!}{
    \begin{tabular}{lllrrrrcccc}
        \toprule
        \textbf{Scale} & \textbf{Dataset} & \textbf{Domain} 
        & \textbf{\#Nodes} & \textbf{\#Edges} & \textbf{Avg. Deg.} & \textbf{\#Feat.}
        & \textbf{Feat. Type} & \textbf{\#Clus.} & \textbf{$\mathcal{H}_e$} & \textbf{$\mathcal{H}_n$} \\
        \midrule
        \multirow{2}{*}{Tiny}
          & Cora & Citation & 2,708 & 10,556 & 3.9 & 1,433 & Textual & 7 & 0.81 & 0.83 \\
          & Photo & Co-purchase & 7,650 & 238,162 & 31.1 & 745 & Textual & 8 & 0.83 & 0.84 \\
        \midrule
        \multirow{2}{*}{Small}
          & HM & Co-purchase & 46,563 & 21,461,990 & 460.9 & 120 & Tabular & 21 & 0.16 & 0.35 \\
          & Flickr & Social & 89,250 & 899,756 & 10.1 & 500 & Textual & 7 & 0.32 & 0.32 \\
        \midrule
        \multirow{3}{*}{Medium}
          & ArXiv & Citation & 169,343 & 1,166,243 & 6.9 & 128 & Textual & 40 & 0.65 & 0.64 \\
          & Reddit & Social & 232,965 & 23,213,838 & 99.6 & 602 & Textual & 41 & 0.78 & 0.81 \\
          & MAG & Citation & 736,389 & 10,792,672 & 14.7 & 128 & Textual & 349 & 0.30 & 0.31 \\
        \midrule
        \multirow{3}{*}{Large}
          & Pokec & Social & 1,632,803 & 44,603,928 & 27.3 & 56 & Tabular & 183 & 0.43 & 0.39 \\
          & Products & Co-purchase & 2,449,029 & 61,859,140 & 25.4 & 100 & Textual & 47 & 0.81 & 0.82 \\
          & WebTopic & Web & 2,890,331 & 24,754,822 & 8.6 & 528 & Tabular & 28 & 0.22 & 0.18 \\
        \bottomrule
    \end{tabular}
    }
\end{table*}

\section{Dataset Details}
\label{app:datasets}

Table~\ref{tab:datasets} summarizes the benchmark datasets used in our experiments.
The datasets span multiple scales, domains, feature types, and homophily regimes.
We use ground-truth labels only for evaluation and do not use them during inference.

\subsection{Tiny Scale Datasets}

\noindent\textbf{Cora}~\citep{GCN}.
Cora is a citation network where nodes represent scientific papers and edges represent citation links.
Node attributes are sparse bag-of-words vectors extracted from paper contents.
The clustering task is to group papers into seven research topics.

\noindent\textbf{Photo}~\citep{Pitfall}.
Photo is an Amazon co-purchase graph.
Nodes represent products, and edges indicate that two products are frequently purchased together.
Node features are bag-of-words representations extracted from product reviews.
The clustering objective is to partition products into eight product categories.

\subsection{Small Scale Datasets}

\noindent\textbf{HM}~\citep{GraphLand}.
HM is derived from the H\&M Kaggle competition.
The graph represents a product co-purchasing network, where nodes are products and edges connect products frequently bought by the same customers.
Unlike most textual benchmarks, HM contains tabular product features, including categorical product metadata and numerical purchasing statistics.
The task is to partition products into 21 product categories.

\noindent\textbf{Flickr}~\citep{GraphSAINT}.
Flickr is an image-sharing social network.
Nodes correspond to images, and edges are constructed based on shared metadata such as location, gallery, or user information.
Node features are 500-dimensional bag-of-words vectors based on user-provided tags.
The clustering task is to group images into seven classes.

\subsection{Medium Scale Datasets}

\noindent\textbf{ArXiv}~\citep{OGB}.
ArXiv is a citation network of Computer Science papers from arXiv.
Nodes represent papers and edges represent citation links.
Node features are 128-dimensional vectors obtained from paper titles and abstracts.
The clustering task is to assign papers to 40 subject areas.

\noindent\textbf{Reddit}~\citep{SAGE}.
Reddit is a social graph constructed from Reddit posts.
Nodes represent posts, and edges connect posts based on common users commenting on both posts.
Node features are text-based representations derived from post-related content.
The clustering target is the subreddit community associated with each post.

\noindent\textbf{MAG}~\citep{OGB}.
MAG is a citation network extracted from the Microsoft Academic Graph.
Following the benchmark setting, we use the homogeneous paper citation graph, where nodes represent papers and edges represent citation relations.
Node features are 128-dimensional textual embeddings derived from paper abstracts.
The clustering task is to group papers into 349 venue categories, making MAG a challenging medium-scale benchmark with a large number of clusters and relatively low homophily.

\subsection{Large Scale Datasets}

\noindent\textbf{Pokec}~\citep{GraphLand}.
Pokec is a social network dataset where nodes represent users and edges represent friendship relations.
Node features are tabular attributes derived from user profile information.
The clustering task is to predict users' regions, resulting in a challenging many-cluster setting with 183 classes.

\noindent\textbf{Products}~\citep{OGB}.
Products is a large-scale Amazon co-purchase network.
Nodes represent products, and edges connect products that are purchased together.
Node features are generated from product descriptions.
The clustering task is to group products into 47 top-level categories.

\noindent\textbf{WebTopic}~\citep{GraphLand}.
WebTopic is a web graph obtained from the Yandex search engine.
Nodes correspond to websites, and directed edges indicate that users followed links from one website to another.
The features are tabular and heterogeneous, including numerical and categorical website attributes.
The clustering goal is to partition websites into 28 topics.
This dataset has low homophily, making it a challenging benchmark for methods that rely heavily on edge-based smoothing.

\section{Baseline Details}
\label{app:baselines}

We follow PyAGC~\citep{PyAGC} and compare \texttt{OFAG} with representative attributed graph clustering methods.
For clarity, we group the baselines into three categories.

\subsection{Traditional and Non-Parametric Methods}

\begin{itemize}[leftmargin=*]
    \item \textbf{KMeans}~\citep{KMeans} clusters the raw node features directly and ignores graph topology. It serves as a basic feature-only baseline.

    \item \textbf{Node2Vec}~\citep{Node2Vec} learns structural node embeddings via biased random walks and then applies $k$-means. It serves as a structure-only baseline.

    \item \textbf{SSGC}~\citep{SSGC} uses simplified spectral graph convolution to propagate features over the graph before clustering. It is a scalable non-parametric encoder followed by $k$-means.

    \item \textbf{SAGSC}~\citep{SAGSC} combines graph smoothing with subspace clustering. It first smooths attributes over the graph and then learns an affinity matrix for spectral clustering.

    \item \textbf{MS2CAG}~\citep{MS2CAG} improves scalable subspace graph clustering by avoiding explicit construction of the full self-expressive matrix and solving a rank-constrained formulation efficiently.
\end{itemize}

\subsection{Deep Decoupled Methods}

\begin{itemize}[leftmargin=*]
    \item \textbf{GAE}~\citep{GAE} learns node embeddings with a graph autoencoder by reconstructing graph structure, followed by clustering on the learned embeddings.

    \item \textbf{DGI}~\citep{DGI} maximizes mutual information between local node representations and a global graph summary through contrastive learning.

    \item \textbf{CCASSG}~\citep{CCASSG} learns graph representations through augmentation-based invariance and feature decorrelation, without requiring explicit negative samples.

    \item \textbf{S3GC}~\citep{S3GC} uses a contrastive objective designed to improve clusterability, with positive relations constructed from graph neighborhoods.

    \item \textbf{NS4GC}~\citep{NS4GC} introduces semantic-aware reliable node similarity to guide representation learning for graph clustering.

    \item \textbf{MAGI}~\citep{MAGI} incorporates modularity maximization into graph contrastive learning, aligning the representation learning objective with community structure.
\end{itemize}

\subsection{Deep Joint Methods}

\begin{itemize}[leftmargin=*]
    \item \textbf{DAEGC}~\citep{DAEGC} uses a graph attention autoencoder and jointly optimizes reconstruction and clustering objectives.

    \item \textbf{DinkNet}~\citep{DinkNet} learns cluster-aware node representations with dilation and shrink losses for scalable graph clustering.

    \item \textbf{MinCut}~\citep{MinCut} optimizes a differentiable relaxation of the normalized min-cut objective with additional regularization to avoid degenerate assignments.

    \item \textbf{DMoN}~\citep{DMoN} optimizes a differentiable modularity objective for graph clustering.

    \item \textbf{Neuromap}~\citep{Neuromap} formulates graph clustering through a differentiable version of the map equation, encouraging partitions that compress graph flow patterns.
\end{itemize}

\begin{table}[h]
\centering
\caption{Key hyperparameters of the CAGM synthetic prior.}
\label{tab:cagm_config}
\begin{adjustbox}{width=0.5\textwidth, center}
\begin{tabular}{llc}
\toprule
\textbf{Category} & \textbf{Parameter} & \textbf{Range / Value} \\
\midrule
\multirow{3}{*}{Episode size}
  & Nodes $N$         & $[512,\;1024]$ \\
  & Feature dim $F$   & $[8,\;128]$    \\
  & Clusters $K$      & $[3,\;10]$     \\
\midrule
\multirow{2}{*}{Latent dims}
  & Attribute latent $d_x$ & $[4,\;32]$ \\
  & Structural latent $d_a$ & $[4,\;32]$ \\
\midrule
\multirow{2}{*}{Cluster proportions}
  & Dirichlet concentration $\eta$ & $[0.5,\;10.0]$ (log-uniform) \\
  & Min.\ cluster size             & $1$ \\
\midrule
\multirow{2}{*}{Graph structure}
  & Edge homophily $h$   & $[0.1,\;0.9]$ (uniform) \\
  & Target avg.\ degree  & $[4,\;64]$ (log-uniform) \\
\midrule
Attributes
  & Max.\ irrelevant feature fraction & $5\%$ \\
\bottomrule
\end{tabular}
\end{adjustbox}
\end{table}

\section{Implementation and Evaluation Details}
\label{app:implementation}

\vpara{OFAG pre-training.}
We pre-train \texttt{OFAG} on synthetic attributed graphs sampled from the proposed prior.
At each step, we sample 8 graphs as a mini-batch.
Training lasts for 10,000 steps, so the model observes 80,000 synthetic graphs in total.
We optimize the hyperspherical clustering objective using AdamW with learning rate $1\times 10^{-4}$ and weight decay $1\times 10^{-4}$.
The SwiFRT hidden dimension is set to 8, and the maximum filter order is set to $L=10$.
The temperature is fixed to $\tau=1.0$, and the dispersion weight is fixed to $\lambda=0.1$.
To monitor the pre-training process, we synthesize a fixed set of 16 synthetic graphs as a validation set before training begins.
At the end of each training step, we compute the loss on this fixed validation set and retain the checkpoint that achieves the lowest validation loss.
Pre-training takes approximately 1.5 hours on a single NVIDIA A800 40G GPU.

\vpara{CAGM prior configuration.}
Table~\ref{tab:cagm_config} summarizes the key hyperparameters of the CAGM generator used during pre-training.
Graph size ($N$), feature dimension ($F$), and number of clusters ($K$) are sampled uniformly from their respective ranges.
Cluster proportions are drawn from a Dirichlet distribution with concentration $\eta$ sampled log-uniformly from $[0.5, 10.0]$: small $\eta$ produces highly imbalanced clusters while large $\eta$ approaches uniform proportions.
Observed node attributes are decoded from Student-$t$ latent vectors through task-specific random linear decoders, and the output dimensions are randomly allocated to five feature types (continuous, binary, count-valued, categorical one-hot, and sparse bag-of-words) with their proportions drawn from a Dirichlet distribution per episode.
At most 5\% of dimensions may be replaced by cluster-irrelevant nuisance features.
The target average degree is sampled log-uniformly from $[4, 64]$, and a binary-search procedure calibrates the global sparsity intercept $\rho$ so that the expected graph density matches the target.
Edge homophily $h$ is sampled uniformly from $[0.1, 0.9]$; the block affinity matrix diagonal is shifted by $s \cdot \mathrm{logit}(h)$ with strength $s \sim \mathrm{Uniform}(0.5, 2.0)$.
The coupling coefficients $\lambda_B \sim \mathrm{Uniform}(0.4, 2.5)$, $\lambda_H$, and $\lambda_X$ control how strongly edges depend on cluster identity, latent structural similarity, and observed attribute similarity, respectively; $\lambda_H$ and $\lambda_X$ are zeroed out with probability 0.25 and 0.35 to produce episodes where graph structure or attributes provide no cross-graph signal.

\vpara{Zero-shot inference.}
After pre-training, \texttt{OFAG} is frozen and directly applied to each benchmark graph.
For a graph $(A,X)$, the model computes normalized node representations in one forward pass.
We then run $k$-means with the ground-truth number of clusters $K$.
No benchmark labels are used for training, fine-tuning, early stopping, hyperparameter selection, or model selection.
The only use of labels is for computing evaluation metrics.

\begin{figure*}[t]
\centering
\subfigure[Homophilous graphs]{\includegraphics[width=0.32\linewidth]{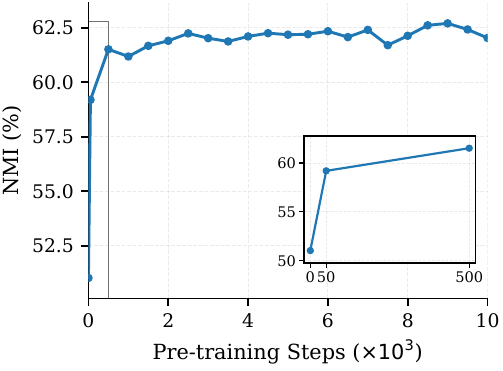}}
\subfigure[Heterophilous graphs]{\includegraphics[width=0.32\linewidth]{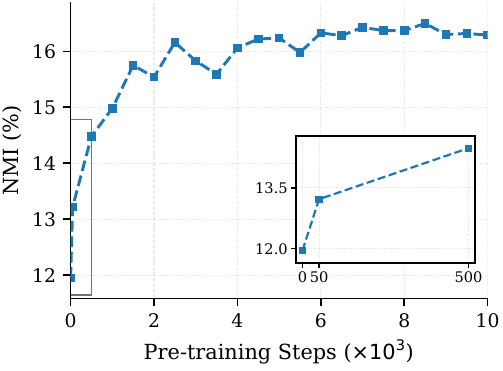}}
\subfigure[All graphs]{\includegraphics[width=0.32\linewidth]{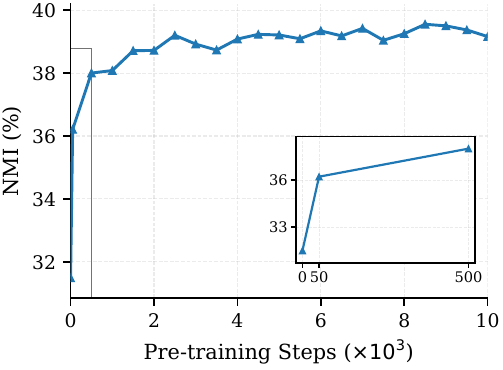}}
\caption{
    NMI of \texttt{OFAG} at different pre-training checkpoints on
    homophilous, heterophilous, and all benchmark graphs.
    Performance improves substantially during the early stage of pre-training
    and subsequently stabilizes, indicating that \texttt{OFAG} rapidly learns
    a transferable clustering inference strategy.
    The inset plots enlarge the early checkpoints to highlight the improvement
    from initialization to the first evaluated pre-training steps.
}
\label{fig:pt_realistic_nmi}
\end{figure*}

\section{Additional Experiments}

\subsection{Pre-training Dynamics on Realistic Graphs}
\label{app:pretrain_dynamics}

Figure~\ref{fig:pt_realistic_nmi} tracks the NMI of \texttt{OFAG} checkpoints when applied directly to the five homophilous datasets, the five heterophilous datasets, and all ten graphs.

\vpara{Two-phase learning.}
All three curves follow a clear two-phase pattern that mirrors the synthetic validation dynamics reported in Section~\ref{sec:pretrain_dynamics}.
NMI improves steeply during the early stage of training, where the gap between the initial checkpoint and those evaluated after the first few hundred steps is noticeably large, and then enters a gradual plateau for the remainder of training.
This indicates that the model rapidly acquires a transferable clustering strategy from the CAGM prior before refining it incrementally.

\vpara{Homophilous vs.\ heterophilous graphs.}
Homophilous graphs consistently yield higher absolute NMI throughout training, reflecting that their smoother attribute-structure correlation is closer to typical prior conditions.
Heterophilous graphs start from a lower baseline but also improve continuously, demonstrating that the prior's diverse homophily spectrum endows the model with non-trivial generalization even to graphs where edge propagation is uninformative.
The relative improvement rate across the two groups is broadly similar, suggesting that the pre-training signal benefits both graph types proportionally.

\vpara{Checkpoint selection.}
Although the NMI trajectories stabilize after the initial rapid gains, they do not settle to a perfectly flat line: minor fluctuations persist across later checkpoints, and the overall best performance on the benchmarks does not always coincide with the final training step.
This motivates our checkpoint selection strategy (i.e., retaining the snapshot with the lowest loss on the fixed synthetic validation set) rather than simply using the last checkpoint.
The validation-loss-based criterion provides a lightweight, label-free proxy that is well-aligned with downstream clustering quality across both graph types.

\subsection{Sensitivity to Filter Order}
\label{app:ablation_L}

SwiFRT builds each node-feature response profile by collecting diffusion outputs at orders $0,\ldots,L$; the choice of $L$ determines the multi-scale context available to the Transformer fusion.
Table~\ref{tab:ablation_L} reports average NMI/ACC as $L$ varies from 2 to 12, with all other settings held fixed.

\begin{table}[h]
\centering
\caption{Sensitivity to filter order $L$ in SwiFRT.}
\label{tab:ablation_L}
\small
\setlength{\tabcolsep}{4.0pt}
\begin{tabular}{ccccccc}
\toprule
\multirow{2}{*}{$L$} & \multicolumn{2}{c}{Homo Avg} & \multicolumn{2}{c}{Hetero Avg} & \multicolumn{2}{c}{All Avg} \\
\cmidrule(lr){2-3}\cmidrule(lr){4-5}\cmidrule(lr){6-7}
 & NMI & ACC & NMI & ACC & NMI & ACC \\
\midrule
2  & 59.54 & 56.47 & 14.90 & 14.49 & 37.22 & 35.48 \\
3  & 61.97 & 60.32 & 15.53 & 15.13 & 38.75 & 37.72 \\
5  & 62.33 & 61.79 & 15.80 & 14.99 & 39.06 & 38.39 \\
8  & 61.77 & 60.78 & 16.05 & 15.20 & 38.91 & 37.99 \\
\textbf{10} & \textbf{62.46} & \textbf{63.01} & \textbf{16.51} & \textbf{16.18} & \textbf{39.49} & \textbf{39.59} \\
12 & 61.92 & 61.42 & 16.39 & 15.63 & 39.16 & 38.53 \\
\bottomrule
\end{tabular}
\end{table}

Two patterns emerge.
First, $L\!=\!2$ causes a marked drop (All NMI $37.22$, vs.\ $38.75$ for $L\!=\!3$), revealing a minimum context threshold: two hops are insufficient to capture multi-scale structural patterns---particularly on heterophilous graphs where same-cluster nodes may not share direct edges and require wider neighborhoods to be jointly identified.
Second, performance stabilizes in the range $L\in[3,12]$ and peaks at $L\!=\!10$, with $L\!=\!12$ yielding only marginal regression ($-0.33$ All NMI).
This plateau indicates that the Transformer can down-weight redundant high-order responses once the useful multi-hop context has been captured, so adding extra filter orders neither helps substantially nor causes over-smoothing collapse.
$L\!=\!10$ offers the best balance and is used throughout.

\subsection{Sensitivity to Dispersion Weight}
\label{app:ablation_lambda}

The training objective $\mathcal{L}=\mathcal{L}_{\mathrm{comp}}+\lambda\,\mathcal{L}_{\mathrm{disp}}$ (Eq.~\ref{eq:expected_full_loss}) balances within-cluster compactness against inter-prototype dispersion.
Table~\ref{tab:ablation_lambda} sweeps $\lambda$ from 0 to 1.0.

\begin{table}[h]
\centering
\caption{Sensitivity to dispersion weight $\lambda$.}
\label{tab:ablation_lambda}
\small
\setlength{\tabcolsep}{4.0pt}
\begin{tabular}{ccccccc}
\toprule
\multirow{2}{*}{$\lambda$} & \multicolumn{2}{c}{Homo Avg} & \multicolumn{2}{c}{Hetero Avg} & \multicolumn{2}{c}{All Avg} \\
\cmidrule(lr){2-3}\cmidrule(lr){4-5}\cmidrule(lr){6-7}
 & NMI & ACC & NMI & ACC & NMI & ACC \\
\midrule
0   & 61.40 & 59.54 & 14.50 & 14.54 & 37.95 & 37.04 \\
\textbf{0.1} & \textbf{62.46} & \textbf{63.01} & \textbf{16.51} & \textbf{16.18} & \textbf{39.49} & \textbf{39.59} \\
0.3 & 61.21 & 59.07 & 16.40 & 15.40 & 38.81 & 37.23 \\
0.5 & 62.43 & 61.96 & 16.33 & 15.38 & 39.38 & 38.67 \\
0.8 & 61.59 & 60.03 & 16.18 & 15.39 & 38.89 & 37.71 \\
1.0 & 62.36 & 61.07 & 16.20 & 15.38 & 39.28 & 38.22 \\
\bottomrule
\end{tabular}
\end{table}

The results reveal two regimes.
Setting $\lambda\!=\!0$ (no dispersion) produces the worst performance, most notably on heterophilous graphs (Hetero NMI $14.50$ vs.\ ${\geq}16.18$ for all $\lambda\!>\!0$), confirming that $\mathcal{L}_{\mathrm{disp}}$ is indispensable for establishing sufficient inter-cluster angular separation, consistent with the component ablation in Section~\ref{sec:ablation_components}.
Once $\lambda\!>\!0$, performance is broadly stable across $\lambda\in[0.1,1.0]$ (All NMI range $38.81$--$39.49$), demonstrating that \texttt{OFAG} is not sensitive to the precise weight as long as dispersion is active.
The mild fluctuations in the homophilous split at larger $\lambda$ reflect an expected tension: an overly strong dispersion term can conflict with the compactness objective by pushing prototypes apart before within-cluster embeddings have fully converged.
$\lambda\!=\!0.1$ achieves the best trade-off and is used as the default.

\subsection{Ablation: Homophily Coverage in the Structural Prior}
\label{app:ablation_homophily}

\vpara{Setup.}
We ablate the range of edge homophily $h$ sampled during pre-training to assess whether broad homophily coverage in the CAGM structural prior is indeed necessary.
We compare four variants: the full range $h\!\in\![0.1,\,0.9]$ used by \texttt{OFAG}, a homophily-only regime $h\!\in\![0.7,\,0.9]$, a near-neutral regime $h\!\in\![0.4,\,0.6]$, and a heterophily-only regime $h\!\in\![0.1,\,0.3]$.
All other pre-training hyperparameters are held fixed; each variant is evaluated zero-shot on the same ten benchmark datasets.

\begin{table}[h]
\centering
\caption{Ablation of the homophily range sampled during pre-training.
Average NMI/ACC over the five homophilous (Homo), five heterophilous (Hetero), and all ten graphs (All).}
\label{tab:ablation_homophily}
\small
\setlength{\tabcolsep}{4.2pt}
\begin{tabular}{lcccccc}
\toprule
\multirow{2}{*}{Homophily range} & \multicolumn{2}{c}{Homo Avg} & \multicolumn{2}{c}{Hetero Avg} & \multicolumn{2}{c}{All Avg} \\
\cmidrule(lr){2-3}\cmidrule(lr){4-5}\cmidrule(lr){6-7}
 & NMI & ACC & NMI & ACC & NMI & ACC \\
\midrule
$[0.1,\,0.9]$ (full, default) & \textbf{62.46} & \textbf{63.01} & \textbf{16.51} & \textbf{16.18} & \textbf{39.49} & \textbf{39.59} \\
$[0.7,\,0.9]$ (homo-only)     & 61.46 & 59.24 & 15.65 & 15.22 & 38.56 & 37.23 \\
$[0.4,\,0.6]$ (near-neutral)  & 57.52 & 54.37 & 15.15 & 14.88 & 36.34 & 34.63 \\
$[0.1,\,0.3]$ (hetero-only)   & 51.62 & 47.34 & 13.13 & 13.26 & 32.37 & 30.30 \\
\bottomrule
\end{tabular}
\end{table}

\vpara{Results and analysis.}
Table~\ref{tab:ablation_homophily} shows a consistent ranking across all three evaluation splits: full coverage dominates every restricted range, confirming that broad homophily diversity in the prior is essential for zero-shot generalization.

\textit{Removing heterophilous pre-training (homo-only, $h\!\in\![0.7,\,0.9]$) penalizes heterophilous benchmarks most.}
Heterophilous NMI drops from 16.51 to 15.65 ($-0.86$) while homophilous NMI declines relatively slightly (62.46$\to$61.46, $-1.00$).
This asymmetry is intuitive: without exposure to heterophilous graphs, the model develops an implicit prior that connected nodes share the same cluster, causing it to over-smooth representations when cross-cluster edges dominate.
The relative resilience on homophilous graphs further confirms that the model is not merely memorizing low-$h$ patterns; it uses the full range to learn an adaptive aggregation strategy.

\textit{Restricting to a near-neutral range ($h\!\in\![0.4,\,0.6]$) imposes a different failure mode.}
Despite covering a balanced mixture of intra- and inter-cluster edges, this variant suffers on both splits (57.52/15.15 NMI), underperforming even homo-only pre-training on homophilous graphs.
The cause is insufficient \emph{signal diversity}: graphs near $h=0.5$ carry a weak structural signal, as intra- and inter-cluster edge probabilities are nearly identical, so the model never learns to exploit strong topology cues.
At inference time, real homophilous graphs present the model with structural signals it has rarely encountered during pre-training.

\textit{Restricting to heterophilous pre-training ($h\!\in\![0.1,\,0.3]$) is the most damaging configuration.}
Homophilous NMI collapses to 51.62 ($-10.84$ vs.\ full range) and overall NMI drops by 7.12 points.
Because the majority of benchmark graphs are homophilous ($\mathcal{H}_e>0.5$), the model is exposed to a distribution that is systematically misaligned with the test domain, and the strongly heterophilous pre-training signal actively encourages the model to route information across cluster boundaries rather than within them.
Notably, heterophilous NMI also falls substantially (13.13 vs.\ 16.51), despite the training regime nominally targeting heterophilous structures.
This seemingly paradoxical result arises because real heterophilous benchmarks are not uniformly extreme: their edge homophily spans a moderate range rather than concentrating near zero.
Pre-training exclusively on strongly heterophilous graphs ($h\!\leq\!0.3$) therefore covers only a narrow slice of the heterophilous spectrum, leaving the model poorly calibrated for graphs with intermediate-to-mild heterophily that appear in the test set.

Taken together, these results validate the design of the CAGM structural prior: \emph{covering the full spectrum $h\!\in\![0.1,\,0.9]$} ensures that the model is neither biased toward smoothing same-cluster signals nor biased toward amplifying cross-cluster signals, enabling reliable zero-shot transfer across diverse real-world homophily regimes.

\section{Time Complexity Analysis}
\label{app:complexity}

We analyze the time complexity of \texttt{OFAG} at inference, decomposing it into the two stages: (1) graph diffusion pre-computation and (2) SwiFRT encoding, followed by $k$-means.
Throughout, $N$ denotes the number of nodes, $|E|$ the number of edges, $F$ the input feature dimension, $L$ the maximum filter order, $d$ the hidden dimension of the Transformer inside SwiFRT, and $K$ the number of clusters.
$L$ and $d$ are fixed hyperparameters ($L=10$, $d=8$ in our implementation).

\vpara{Graph diffusion pre-computation.}
SwiFRT requires the diffusion stack $\{\hat{A}^\ell X\}_{\ell=0}^{L}$.
Each multiplication $\hat{A}^\ell X = \hat{A}\cdot(\hat{A}^{\ell-1}X)$ is a sparse-dense matrix product with cost $O(|E|\cdot F)$.
Computing all $L$ powers therefore costs
\[
    O(L\cdot |E|\cdot F).
\]
For sparse graphs where $|E|=O(N)$ (e.g., citation or social networks), this reduces to $O(L\cdot N\cdot F)$, i.e., linear in the graph size.

\vpara{SwiFRT encoding.}
For each node--feature pair $(i,f)$, the Transformer processes a token sequence of length $L+2$ (one CLS token plus $L+1$ filter-response tokens).
Self-attention within this fixed-length sequence costs $O((L+2)^2 \cdot d) = O(L^2 d)$ per pair.
Summing over all $N\cdot F$ pairs gives
\[
    O(N\cdot F\cdot L^2\cdot d).
\]
Since $L$ and $d$ are constants, this simplifies to $O(N\cdot F)$, i.e., linear in both the number of nodes and the feature dimension.

\vpara{Post-encoding $k$-means.}
Standard $k$-means on $N$ embeddings of dimension $F$ for $T$ iterations costs $O(N\cdot F\cdot K\cdot T)$.
In practice $K$ and $T$ are small, so this term is dominated by the encoding stage for large graphs.

\vpara{Total inference complexity.}
Combining the above, the end-to-end inference complexity of \texttt{OFAG} is
\[
    O\bigl(L\cdot|E|\cdot F \;+\; N\cdot F\cdot L^2 d\bigr)
    \;=\; O\bigl(|E|\cdot F \;+\; N\cdot F\bigr)
    \;=\; O\bigl((|E|+N)\cdot F\bigr),
\]
where the simplification treats $L$ and $d$ as constants.
For sparse graphs this is $O(N\cdot F)$---linear in the graph size---which explains the order-of-magnitude speedup over iterative optimization-based baselines observed in Table~\ref{tab:efficiency}.

\vpara{Comparison with baselines.}
Dataset-specific clustering methods additionally incur the cost of per-graph optimization.
For example, GNN-based methods such as DinkNet and S3GC perform $T_{\mathrm{opt}}$ gradient steps, each of cost $O(|E|\cdot d_{\mathrm{model}})$, giving a total optimization overhead of $O(T_{\mathrm{opt}}\cdot|E|\cdot d_{\mathrm{model}})$ on top of the encoding.
Contrastive learning methods (e.g., MAGI, NS4GC) introduce an additional quadratic term: computing all pairwise node similarities or constructing a full negative-sample matrix costs $O(N^2)$ per step, which becomes prohibitive for large batch sizes.
Hyperparameter search multiplies all of these costs further.
\texttt{OFAG} eliminates both the per-graph optimization overhead and the quadratic contrastive term entirely, reducing inference to a single forward pass followed by $k$-means.

\end{document}